\documentclass[11pt]{article}

\usepackage{amsmath, amssymb, amsthm, fullpage, hyperref}
\usepackage{algorithmic, algorithm}
\usepackage{graphicx}
\usepackage{psfrag}
\usepackage{fullpage}

\newcommand{\F}{\ensuremath{\mathcal{F}}}
\newcommand{\G}{\ensuremath{\mathcal{G}}}

 \newtheorem{theorem}{Theorem}
\newtheorem{lemma}{Lemma} 
\newtheorem{definition}{Definition}
\newcommand{\E}{\ensuremath{\mathbb{E}}}
\renewcommand{\P}{\ensuremath{\mathbb{P}}}
\newcommand{\x}[1]{\ensuremath{x(#1)}}
\newcommand{\KL}[2]{\ensuremath{D(#1 \, \| #2)}}

\newcommand{\parama}{\ensuremath{\alpha}}
\newcommand{\paramb}{\ensuremath{\beta}}

\newcommand{\set}{\ensuremath{\Sset}}
\newcommand{\meth}{\ensuremath{\mathcal{M}}}
\newcommand{\info}{\ensuremath{\phi}}
\renewcommand{\O}{\ensuremath{\mathcal{O}}}
\newcommand{\acc}{\ensuremath{\epsilon}}
\newcommand{\bestacc}{\epsilon^*}
\newcommand{\err}{\ensuremath{\delta}}
\renewcommand{\dim}{\ensuremath{d}}
\newcommand{\inrad}{\ensuremath{r}}
\newcommand{\Fconv}{\ensuremath{\mathcal{F}_{\operatorname{\scriptsize{cv}}}}}

\newcommand{\Fstrong}{\ensuremath{\mathcal{F}_{\operatorname{\scriptsize{scv}}}}}
    \newcommand{\Fsparse}{\ensuremath{\mathcal{F}_{\operatorname{\scriptsize{sp}}}}}
    \newcommand{\disc}{\ensuremath{\rho}}
    \newcommand{\errfun}{\ensuremath{\psi}}
    \newcommand{\esta}[1]{\ensuremath{\widehat{\parama}(#1)}}
    
    \newcommand{\fp}{\ensuremath{f^+}}
    \newcommand{\fn}{\ensuremath{f^-}}
    \newcommand{\half}{\ensuremath{\frac{1}{2}}}
    
    \newcommand{\coins}{\ensuremath{\Theta}}

    \newcommand{\packset}{\ensuremath{\mathcal{V}}}
    
    \newcommand{\indicator}{\ensuremath{\mathbb{I}}}
    
    \newcommand{\truth}{\ensuremath{\phi}}
    
    \newcommand{\oraclev}{\ensuremath{\widehat{f}}}
    \newcommand{\oracleg}{\ensuremath{\widehat{z}}}

    \newcommand{\lipexp}{\ensuremath{p}}
    \newcommand{\lipexpdual}{\ensuremath{q}}
    \newcommand{\grad}{\ensuremath{\nabla}}
    \newcommand{\ind}{\ensuremath{\mathbb{I}}} \def\elem{x}
    \def\fun{f}

\newcommand{\ConvSet}{\ensuremath{\mathbb{S}}}

 \def\FunSet{\mathcal{F}}
 \def\InfoSet{\mathcal{I}} \def\real{\mathbb{R}}

\def\methset{\mathbb{M}}
\def\defn{:=}
\def\funclassset{\mathcal{G}_{\operatorname{base}}}

\newenvironment{carlist}
 {\begin{list}{$\bullet$}
 {\setlength{\topsep}{0in} \setlength{\partopsep}{0in}
  \setlength{\parsep}{0in} \setlength{\itemsep}{\parskip}
  \setlength{\leftmargin}{0.1in} \setlength{\rightmargin}{0.08in}
  \setlength{\listparindent}{002in} \setlength{\labelwidth}{0.08in}
  \setlength{\labelsep}{0.1in} \setlength{\itemindent}{0in}}}
 {\end{list}}

\newcommand{\bcar}{\begin{carlist}}
\newcommand{\ecar}{\end{carlist}}

\newcommand{\elemstar}{\ensuremath{\elem^*}}
\newcommand{\pdim}{\ensuremath{d}}
\newcommand{\Hamm}{\ensuremath{\Delta_H}}

\newcommand{\paramtrue}{\ensuremath{\parama^*}}

\newcommand{\NemYu}{NY}

\newcommand{\subdiff}{\ensuremath{\partial}}

\newcommand{\Exs}{\ensuremath{\mathbb{E}}}

\newcommand{\pval}{\ensuremath{p}}
\newcommand{\qval}{\ensuremath{q}}

\newcommand{\inprod}[2]{\ensuremath{\langle #1 , \, #2 \rangle}}

\newcommand{\StochOrac}[1]{\ensuremath{\mathbb{O}_{\pval, #1}}}
\newcommand{\StochOractwoparam}[2]{\ensuremath{\mathbb{O}_{#1, #2}}}

\newcommand{\Fclass}{\ensuremath{\mathcal{F}}}

\newcommand{\numobs}{\ensuremath{n}}

\newcommand{\Sset}{\ensuremath{\mathbb{S}}}

\newcommand{\Ball}{\ensuremath{\mathbb{B}}}

\newcommand{\Lval}{\ensuremath{L}}

\newcommand{\mprob}{\ensuremath{\mathbb{P}}}
\newcommand{\qprob}{\ensuremath{\mathbb{Q}}}

\newcommand{\kdim}{\ensuremath{k}}

\newcommand{\order}{\ensuremath{\mathcal{O}}}

\newcommand{\alhat}{\ensuremath{\widehat{\alpha}}}

\newcommand{\Data}[1]{\ensuremath{\phi(x_1^T; #1)}}
\long\def\comment#1{}

\newcommand{\oracfun}{\ensuremath{\widehat{g}}}

\newcommand{\oracfuna}[1]{\ensuremath{\oracfun_{#1,A}}}
\newcommand{\oracsuba}[1]{\ensuremath{\oracleg_{#1,A}}}
\newcommand{\oracfunb}[1]{\ensuremath{\oracfun_{#1,B}}}
\newcommand{\oracsubb}[1]{\ensuremath{\oracleg_{#1,B}}}

\newcommand{\prox}{\Phi}
\def\norm#1{\|#1\|}
\def\dnorm#1{\|#1\|_*}

\newcommand{\Rparam}{\ensuremath{\alpha^*}}
\newcommand{\rparam}{\ensuremath{\alpha}}
\newcommand{\kull}[2]{\KL{#1}{#2}}

\newcommand{\Mysub}{\ensuremath{U}}

\newcommand{\strongcon}{\ensuremath{\gamma}}

\newcommand{\qpar}{\ensuremath{q}}

\newcommand{\HackoRama}{\ensuremath{\order(\dim^{1-\delta})}}

\newcommand{\xstar}{\ensuremath{x^*}}

\newcommand{\Prox}{\prox}

\newcommand{\widgraph}[2]{\includegraphics[keepaspectratio,width=#1]{#2}}

\newcommand{\interior}{\ensuremath{\mbox{int}}}

\makeatletter
\long\def\@makecaption#1#2{
        \vskip 0.8ex
        \setbox\@tempboxa\hbox{\small {\bf #1:} #2}
        \parindent 1.5em  
        \dimen0=\hsize
        \advance\dimen0 by -3em
        \ifdim \wd\@tempboxa >\dimen0
                \hbox to \hsize{
                        \parindent 0em
                        \hfil 
                        \parbox{\dimen0}{\def\baselinestretch{0.96}\small
                                {\bf #1.} #2
                                } 
                        \hfil}
        \else \hbox to \hsize{\hfil \box\@tempboxa \hfil}
        \fi
        }
\makeatother

\begin{document}

\begin{center}

  {\bf{\LARGE{ Information-theoretic lower bounds on the oracle
        complexity of stochastic convex optimization}}}

\vspace*{.2in}

\begin{tabular}{ccc}
Alekh Agarwal$^1$ & &  Peter L. Bartlett$^{1,2,3}$ \\
{\texttt{alekh@cs.berkeley.edu}} & &
{\texttt{peter@berkeley.edu}} \\
& & \\
Pradeep Ravikumar$^4$ & & Martin J. Wainwright$^{1,2}$ \\
\texttt{pradeepr@cs.utexas.edu} & &
       {\texttt{wainwrig@stat.berkeley.edu}}
\end{tabular}

\vspace*{.2in}

\begin{tabular}{c}
Department of Electrical Engineering and Computer Sciences$^1$ \\
Department of Statistics$^2$ \\
UC Berkeley, Berkeley, CA  
\end{tabular}

\vspace*{.05in}

\begin{tabular}{ccc}
Mathematical Sciences$^3$ & &  Department of Computer Sciences$^4$   \\
QUT, Brisbane, Australia & & UT Austin, Austin, TX
\end{tabular}

\vspace*{.2in}

\today

\begin{abstract}
  Relative to the large literature on upper bounds on complexity of
  convex optimization, lesser attention has been paid to the
  fundamental hardness of these problems.  Given the extensive use of
  convex optimization in machine learning and statistics, gaining an
  understanding of these complexity-theoretic issues is important. In
  this paper, we study the complexity of stochastic convex
  optimization in an oracle model of computation. We improve upon
  known results and obtain tight minimax complexity estimates for
  various function classes. 
\end{abstract}

\end{center}


\section{Introduction}
\label{sec:intro}

Convex optimization forms the backbone of many algorithms for
statistical learning and estimation. Given that many statistical
estimation problems are large-scale in nature---with the problem
dimension and/or sample size being large---it is essential to make
efficient use of computational resources.  Stochastic optimization
algorithms are an attractive class of methods, known to yield
moderately accurate solutions in a relatively short
time~\cite{bottou08tradeoff}.  Given the popularity of such stochastic
optimization methods, understanding the fundamental computational
complexity of stochastic convex optimization is thus a key issue for
large-scale learning.  A large body of literature is devoted to
obtaining rates of convergence of specific procedures for various
classes of convex optimization problems. A typical outcome of such
analysis is an upper bound on the error---for instance, gap to the
optimal cost---as a function of the number of iterations. Such
analyses have been performed for many standard optimization
algorithms, among them gradient descent, mirror descent, interior
point programming, and stochastic gradient descent, to name a few. We
refer the reader to various standard texts on optimization
(e.g.,~\cite{Boyd02,Bertsekas_nonlin,nesterov2004book}) for further
details on such results.

On the other hand, there has been relatively little study of the
inherent complexity of convex optimization problems. To the best of
our knowledge, the first formal study in this area was undertaken in
the seminal work of Nemirovski and Yudin~\cite{yudin83book}, hereafter
referred to as \NemYu.  One obstacle to a classical
complexity-theoretic analysis, as these authors observed, is that of
casting convex optimization problems in a Turing Machine model. They
avoided this problem by instead considering a natural oracle model of
complexity, in which at every round the optimization procedure queries
an oracle for certain information on the function being
optimized. This information can be either noiseless or noisy,
depending on whether the goal is to lower bound the oracle complexity
of deterministic or stochastic optimization algorithms. Working within
this framework, the authors obtained a series of lower bounds on the
computational complexity of convex optimization problems, both in
deterministic and stochastic settings.  In addition to the original
text \NemYu~\cite{yudin83book}, we refer the interested reader to the
book by Nesterov~\cite{nesterov2004book}, and the lecture notes by
Nemirovski~\cite{nemirovskinotes} for further background.

In this paper, we consider the computational complexity of stochastic
convex optimization within this oracle model.
In particular, we improve upon the work of \NemYu~\cite{yudin83book}
for stochastic convex optimization in two ways. First, our lower
bounds have an improved dependence on the dimension of the space. In
the context of statistical estimation, these bounds show how the
difficulty of the estimation problem increases with the number of
parameters.  Second, our techniques naturally extend to give sharper
results for optimization over simpler function classes. We show that
the complexity of optimization for strongly convex losses is smaller
than that for convex, Lipschitz losses. Third, we show that for a
fixed function class, if the set of optimizers is assumed to have
special structure such as sparsity, then the fundamental complexity of
optimization can be significantly smaller.  All of our proofs exploit
a new notion of the discrepancy between two functions that appears to
be natural for optimization problems.  They involve a reduction from a
statistical parameter estimation problem to the stochastic
optimization problem, and an application of information-theoretic
lower bounds for the estimation problem. We note that special cases of
the first two results in this paper appeared in the extended
abstract~\cite{agarwal2009oracle}, and that a related study was
independently undertaken by Raginsky and Rakhlin~\cite{RagRak11}.

The remainder of this paper is organized as follows.  We begin in
Section~\ref{sec:prev} with background on oracle complexity, and a
precise formulation of the problems addressed in this paper.
Section~\ref{sec:results} is devoted to the statement of our main
results, and discussion of their consequences.  In
Section~\ref{sec:proofs}, we provide the proofs of our main results,
which all exploit a common framework of four steps.  More technical
aspects of these proofs are deferred to the appendices. \\
~
\vspace*{.05in}
~
\noindent \paragraph{Notation:}

For the convenience of the reader, we collect here some notation used
throughout the paper.  For $\pval \in [1, \infty]$, we use
$\|x\|_\pval$ to denote the $\ell_\pval$-norm of a vector $x \in
\real^\pval$, and we let $\qval$ denote the conjugate exponent,
satisfying $\frac{1}{\pval} + \frac{1}{\qval} = 1$.  For two
distributions $\mprob$ and $\qprob$, we use $\KL{\mprob}{\qprob}$ to
denote the Kullback-Leibler (KL) divergence between the
distributions. The notation $\indicator(A)$ refers to the 0-1 valued
indicator random variable of the set $A$. For two vectors
$\parama,\paramb \in \{-1,+1\}^d$, we define the Hamming distance
$\Hamm(\parama, \paramb) \defn \sum_{i=1}^\dim \indicator[\parama_i
  \ne \paramb_i]$.  Given a convex function $f:\real^d \rightarrow
\real$, the subdifferential of $f$ at $x$ is the set $\partial f(x)
\defn \{z \in \real^d \; \mid \; f(y) \geq f(x) + \inprod{z}{y-x}
\quad \mbox{for all $y \in \real^d$}\}$.


\section{Background and problem formulation}
\label{sec:prev}

We begin by introducing background on the oracle model of convex
optimization, and then turn to a precise specification of the problem
to be studied.

\subsection{Convex optimization in the oracle model}
\label{sec:setup}

Convex optimization is the task of minimizing a convex function $f$
over a convex set $\ConvSet\subseteq \mathbb{R}^\dim$.  Assuming that
the minimum is achieved, it corresponds to computing an element
$\elemstar_\fun$ that achieves the minimum---that is, an element
\mbox{$\elemstar_{\fun} \in \arg \min_{\elem \in \ConvSet}
  \fun(\elem)$.}  An \emph{optimization method} is any procedure that
solves this task, typically by repeatedly selecting values from
$\ConvSet$.  For a given class of optimization problems, our primary
focus in this paper is to determine lower bounds on the computational
cost, as measured in terms of the number of (noisy) function and
subgradient evaluations, required to obtain an $\epsilon$-optimal
solution to any optimization problem within the class.

More specifically, we follow the approach of Nemirovski and
Yudin~\cite{yudin83book}, and measure computational cost based on the
oracle model of optimization. The main components of this model are an
\emph{oracle} and an \emph{information set}. An \emph{oracle} is a
(possibly random) function $\info: \ConvSet \mapsto \InfoSet$ that
answers any query $\elem \in \ConvSet$ by returning an element
$\info(\elem)$ in an information set $\InfoSet$. The information set
varies depending on the oracle; for instance, for an exact oracle of
$m^{th}$ order, the answer to a query $x_t$ consists of $f(x_t)$ and
the first $m$ derivatives of $f$ at $x_t$. For the case of stochastic
oracles studied in this paper, these values are corrupted with
zero-mean noise with bounded variance. We then measure the
computational labor of any optimization method as the number of
queries it poses to the oracle.

In particular, given a positive integer $T$ corresponding to the
number of iterations, an optimization method $\meth$ designed to
approximately minimize the convex function $f$ over the convex set
$\ConvSet$ proceeds as follows. At any given iteration $t = 1,
\dots,T$, the method $\meth$ queries at $x_t \in \ConvSet$, and the
oracle reveals the information $\info(x_t, f)$.  The method then uses
the information $\{\info(x_1,f),\dots,\info(x_t,f)\}$ to decide at
which point $x_{t+1}$ the next query should be made.  For a given
oracle function $\info$, let $\methset_T$ denote the class of all
optimization methods $\meth$ that make $T$ queries according to the
procedure outlined above.  For any method $\meth \in \methset_T$, we
define its error on function $f$ after $T$ steps as
\begin{equation}
\label{EqnDefnSpecificError}
  \acc_T(\meth, f, \ConvSet, \info) \defn f(\elem_T) - \min_{\elem \in
    \ConvSet} f(\elem) \; = \; f(\elem_T) - f(\elem^*_f),
\end{equation}
where $x_T$ is the method's query at time $T$.  Note that by
definition of $\elem^*_f$ as a minimizing argument, this error is a
non-negative quantity.

When the oracle is stochastic, the method's query $x_T$ at time $T$ is
itself random, since it depends on the random answers provided by the
oracle.  In this case, the optimization error $\acc_T(\meth,
f,\ConvSet,\info)$ is also a random variable.  Accordingly, for the case
of stochastic oracles, we measure the accuracy in terms of the
expected value $\E_{\info} [\acc_T(\meth, f,\ConvSet,\info)]$, where the
expectation is taken over the oracle randomness.  Given a class of
functions $\FunSet$ defined over a convex set $\ConvSet$ and a class
$\methset_T$ of all optimization methods based on $T$ oracle queries,
we define the minimax error
\begin{equation}
\label{EqnDefnMinimaxError}
\bestacc_T(\F,\ConvSet; \info) \defn \inf_{\meth \in \methset_T}\sup_{f
  \in \F} \E_{\info} [\acc_T(\meth, f,\ConvSet,\info)].
\end{equation}
In the sequel, we provide results for particular classes of oracles.
So as to ease the notation, when the oracle $\info$ is clear from the
context, we simply write $\bestacc_T(\F, \ConvSet)$.


\subsection{Stochastic first-order oracles}

In this paper, we study stochastic oracles for which the information
set $\InfoSet \subset \real \times \real^\dim$ consists of pairs of
noisy function and subgradient evaluations.  More precisely, we have:
\begin{definition}
\label{DefnStochOrac}
For a given set $\ConvSet$ and function class $\Fclass$, the class of
first-order stochastic oracles consists of random mappings $\phi: S
\times \Fclass \rightarrow \InfoSet$ of the form \mbox{$\phi(x,f) =
  (\oraclev(x), \, \oracleg(x))$} such that
\begin{equation}
\label{EqnOracCondition}
\Exs[\oraclev(x)] = f(x), \quad \quad \Exs [\oracleg(x)] \in \subdiff
f(x), \quad \mbox{ and }
\quad \Exs \big[ \, \|\oracleg(x)\|_\pval^2  
  \, \big] \; \leq \sigma^2.
\end{equation}
\end{definition}
\noindent 
We use $\StochOrac{\sigma}$ to denote the class of all stochastic first-order
oracles with parameters $(\pval, \sigma)$.  Note that the first two
conditions imply that $\oraclev(x)$ is an unbiased estimate of the
function value $f(x)$, and that $\oracleg(x)$ is an unbiased estimate
of a subgradient $z \in \partial f(x)$.  When $f$ is actually
differentiable, then $\oracleg(x)$ is an unbiased estimate of the
gradient $\nabla f(x)$.  The third condition in
equation~\eqref{EqnOracCondition} controls the ``noisiness'' of the
subgradient estimates in terms of the $\ell_\pval$-norm. \\

Stochastic gradient methods are a widely used class of algorithms that
can be understood as operating based on information provided by a
stochastic first-order oracle.  As a particular example, consider a
function of the separable form $f(x) = \frac{1}{\numobs}
\sum_{i=1}^\numobs h_i(x)$, where each $h_i$ is differentiable.
Functions of this form arise very frequently in statistical problems,
where each term $i$ corresponds to a different sample and the overall
cost function is some type of statistical loss (e.g., maximum
likelihood, support vector machines, boosting etc.)  The natural
stochastic gradient method for this problem is to choose an index $i
\in \{1, 2, \ldots, \numobs \}$ uniformly at random, and then to
return the pair $(h_i(x), \nabla h_i(x))$.  Taking averages over the
randomly chosen index $i$ yields $\frac{1}{\numobs} \sum_{i=1}^\numobs
h_i(x) = f(x)$, so that $h_i(x)$ is an unbiased estimate of $f(x)$,
with an analogous unbiased property holding for the gradient of $h_i(x)$.

\subsection{Function classes of interest}

We now turn to the classes $\F$ of convex functions for which we study
oracle complexity. In all cases, we consider real-valued convex
functions defined over some convex set $\ConvSet$.  We assume without
loss of generality that $\ConvSet$ contains an open set around $0$,
and many of our lower bounds involve the maximum radius $\inrad =
\inrad(\Sset) > 0$ such that
\begin{align}
\ConvSet \; & \supseteq \; \Ball_\infty(\inrad) \, \defn \, \big \{ x \in
\real^\dim \, \mid \, \|x\|_\infty \leq \inrad \big \}.
\end{align}

\noindent Our first class consists of \emph{convex Lipschitz
  functions:}
\begin{definition}
\label{DefnConvLip}
For a given convex set $\ConvSet \subseteq \real^\dim$ and parameter
$\pval \in [1,\infty]$, the class $\Fconv(\ConvSet, L, \pval)$ consists
of all convex functions $f: \ConvSet \rightarrow \real$ such that 
\begin{align}
\label{EqnFlip}
\big|f(x) - f(y) \big| & \leq L \; \|x-y\|_{\qval} \qquad \mbox{for
  all $x, y \in \ConvSet$,}
\end{align}
where $\frac{1}{\qval} = 1 - \frac{1}{\pval}$.
\end{definition} 

We have defined the Lipschitz condition~\eqref{EqnFlip} in terms of
the conjugate exponent $\qval \in [1, \infty]$, defined by the
relation $\frac{1}{\qval} = 1 - \frac{1}{\pval}$.  To be clear, our
motivation in doing so is to maintain consistency with our definition
of the stochastic first-order oracle, in which we assumed that $\Exs
\big[ \, \|\oracleg(x)\|_\pval^2 \, \big] \; \leq \sigma^2$.  We note
that the Lipschitz condition~\eqref{EqnFlip} is equivalent to the
condition
\begin{align*}
\norm{z}_p \leq L\quad \forall z \in \partial f(x),~~\mbox{and for
  all } x \in \interior(\set).  
\end{align*}
If we consider the case of a differentiable function $f$, the
unbiasedness condition in Definition~\ref{DefnStochOrac} implies that
\begin{equation*}
\|\grad f(x)\|_\pval  = \| \Exs[\oracleg(x)]\|_\pval \;
\stackrel{(a)}{\leq} \; \Exs \|\oracleg(x)\|_\pval \;
\stackrel{(b)}{\leq} \; \sqrt{\Exs \|\oracleg(x)\|^2_\pval} \; \leq \;
\sigma,
\end{equation*}
where inequality (a) follows from the convexity of the
$\ell_\pval$-norm and Jensen's inequality, and inequality (b) is a
result of Jensen's inequality applied to the concave function
$\sqrt{x}$.  This bound implies that $f$ must be Lipschitz with
constant at most $\sigma$ with respect to the dual $\ell_\qval$-norm.
Therefore, we necessarily must have $L \leq \sigma$, in order for the
function class from Definition~\ref{DefnConvLip} to be consistent with
the stochastic first-order oracle. \\

\noindent A second function class consists of strongly convex
functions, defined as follows:
\begin{definition}
\label{DefnStrongConvex}
For a given convex set $\ConvSet \subseteq \real^\dim$ and parameter
$\pval \in [1,\infty]$, the class $\Fstrong(\ConvSet, \pval; L, \strongcon)$
consists of all convex functions $f: \ConvSet \rightarrow \real$ such
that the Lipschitz
condition~\eqref{EqnFlip} holds, and such that $f$ satisfies the
$\ell_2$-strong convexity condition
\begin{align}
f\left(\alpha x + (1-\alpha)y\right) & \geq \alpha f(x) +
(1-\alpha)f(y) + \alpha(1-\alpha)\frac{\strongcon^2}{2} \|x-y\|_2^2
\qquad \mbox{for all $x, y \in \ConvSet$.}
\end{align}
\end{definition}

In this paper, we restrict our attention to the case of strong
convexity with respect to the $\ell_2$-norm.  (Similar results on the
oracle complexity for strong convexity with respect to different norms
can be obtained by straightforward modifications of the arguments
given here).  For future reference, it should be noted that the
Lipschitz constant $L$ and strong convexity constant $\strongcon$
interact with one another.  In particular, whenever $\ConvSet \subset
\real^\dim$ contains the $\ell_\infty$-ball of radius $\inrad$, the
Lipschitz $L$ and strong convexity $\strongcon$ constants must satisfy
the inequality
\begin{align}
\label{EqnLipKap}
\frac{L}{\strongcon^2} & \; \geq \; \frac{\inrad}{4} \,
\dim^{1/\lipexp}.
\end{align} 
In order to establish this inequality, we note that strong convexity
condition with $\alpha = 1/2$ implies that
\begin{align*}
  \frac{\strongcon^2}{8} & \leq \frac{
    2f\left(\frac{x+y}{2}\right) - f(x) - f(y)}{2\|x-y\|^2_2} \; \leq \frac{L \|x -
    y \|_\lipexpdual}{2\|x-y\|_2^2}
\end{align*}
We now choose the pair $x, y \in \ConvSet$ such that $\|x-y\|_\infty =
\inrad$ and $\|x - y\|_2 = \inrad\sqrt{\dim}$.  Such a choice is
possible whenever $\ConvSet$ contains the $\ell_\infty$ ball of radius
$\inrad$.  Since we have $\|x-y\|_\lipexpdual \leq d^{1/\lipexpdual}
\|x -y\|_\infty$, this choice yields $\frac{\strongcon^2}{4} \leq \frac{
  L d^{\frac{1}{\lipexpdual} - 1}}{\inrad}$, which establishes the
claim~\eqref{EqnLipKap}. \\
 
\vspace*{.1in}

As a third example, we study the oracle complexity of optimization
over the class of convex functions that have sparse minimizers.  This
class of functions is well-motivated, since a large body of
statistical work has studied the estimation of vectors, matrices and
functions under various types of sparsity constraints.  A common theme
in this line of work is that the ambient dimension $\dim$ enters
only logarithmically, and so has a mild effect.  Consequently, it
is natural to investigate whether the complexity of optimization
methods also enjoys such a mild dependence on ambient dimension
under sparsity assumptions.

\vspace*{.1in}

For a vector $x \in \real^\dim$, we use $\|x\|_0$ to denote the number
of non-zero elements in $x$.  Recalling the set $ \Fconv(\ConvSet, L,
\pval)$ from Definition~\ref{DefnConvLip}, we now define a class of
Lipschitz functions with sparse minimizers.

\begin{definition}
\label{DefnSparseFunc}
For a convex set $\ConvSet \subset \real^d$ and positive integer
$\kdim \leq \lfloor d/2 \rfloor$, let
\begin{align}
\Fsparse(\kdim; \ConvSet, L) & \defn \big \{ f \in \Fconv(\ConvSet, L,
\infty) \, \mid \, \exists \quad x^* \in \arg \min_{x \in \ConvSet}
f(x) \quad \mbox{satisfying $\|x^*\|_0 \leq \kdim$.}  \big \}
\end{align}
be the class of all convex functions that are $L$-Lipschitz in the
$\ell_\infty$-norm, and have at least one $\kdim$-sparse optimizer.
\end{definition}
\noindent We frequently use the shorthand notation $\Fsparse(\kdim)$
when the set $\ConvSet$ and parameter $L$ are clear from context.


\section{Main results and their consequences}
\label{sec:results}

With the setup of stochastic convex optimization in place, we are now
in a position to state the main results of this paper, and to discuss
some of their consequences.  As previously mentioned, a subset of our
results assume that the set $\ConvSet$ contains an $\ell_\infty$ ball
of radius $\inrad = \inrad(\ConvSet)$. Our bounds scale with $\inrad$,
thereby reflecting the natural dependence on the size of the set
$\ConvSet$. Also, we set the oracle second moment bound $\sigma$ to be
the same as the Lipschitz constant $\Lval$ in our results.

\subsection{Oracle complexity for convex Lipschitz functions}

We begin by analyzing the minimax oracle complexity of optimization
for the class of bounded and convex Lipschitz functions $\Fconv$ from
Definition~\ref{DefnConvLip}.

\begin{theorem}
\label{ThmConvex}
Let $\ConvSet \subset \real^\dim$ be a convex set such that $\ConvSet
\supseteq \Ball_\infty(\inrad)$ for some $\inrad > 0$.  Then for a
universal constant $c_0 > 0$, the minimax oracle complexity over the
class $\Fconv(\ConvSet, \Lval, \pval)$ satisfies the following lower
bounds:
\begin{itemize}
\item[(a)] For $1 \leq \pval \leq 2$,
\begin{align}
\label{EqnConvLip12}
\sup_{\phi \in \StochOrac{\Lval}} \bestacc_T(\Fconv,\ConvSet; \info) &
\geq \min\left\{c_0\Lval \; \inrad \;
\sqrt{\frac{d}{T}},\frac{\Lval\inrad}{144} \right\}.
\end{align}
\item[(b)] For $\pval > 2$,
\begin{align}
\label{EqnConvLipG2}
\sup_{\phi \in \StochOrac{\Lval}} \bestacc_T(\Fconv,\ConvSet; \info) &
\geq \min\left \{c_0\Lval \; \inrad \;
\frac{d^{1-\frac{1}{\pval}}}{\sqrt{T}}, \frac{\Lval
  d^{1-1/\pval}\inrad}{72} \right\}.
\end{align}
\end{itemize}
\end{theorem}

\paragraph{Remarks:} 
Nemirovski and Yudin~\cite{yudin83book} proved the lower bound $\Omega
\big(\frac{1}{\sqrt{T}} \big)$ for the function class $\Fconv$, in the
special case that $\ConvSet$ is the unit ball of a given norm, and the
functions are Lipschitz in the corresponding \emph{dual norm}.  For $p
\geq 2$, they established the minimax optimality of this
dimension-independent result by appealing to a matching upper bound
achieved by the method of mirror descent.  In contrast, here we do not
require the two norms---namely, that constraining the set $\ConvSet$
and that for the Lipschitz constraint---to be dual to one other;
instead, we give give lower bounds in terms of the largest
$\ell_\infty$ ball contained within the constraint set $\ConvSet$. As
discussed below, our bounds do include the results for the dual
setting of past work as a special case, but more generally, by
examining the relative geometry of an arbitrary set with respect to
the $\ell_\infty$ ball, we obtain results for arbitrary sets.  (We
note that the $\ell_\infty$ constraint is natural in many optimization
problems arising in machine learning settings, in which upper and
lower bounds on variables are often imposed.)  Thus, in contrast to
the past work of NY on stochastic optimization, our analysis gives
sharper dimension dependence under more general settings. It also
highlights the role of the geometry of the set $\ConvSet$ in
determining the oracle complexity.

In general, our lower bounds cannot be improved, and hence specify the
optimal minimax oracle complexity.  We consider here some examples to
illustrate their sharpness. Throughout we assume that $T$ is large
enough to ensure that the $1/\sqrt{T}$ term attains the lower bound
and not the $L/144$ term.  (This condition is reasonable given our
goal of understanding the rate as $T$ increases, as opposed to the
transient behavior over the first few iterations.)
\begin{enumerate}
\item[(a)] We start from the special case that has been primarily
  considered in past works. We consider the class $\Fconv(\Ball_q(1),
  L, p)$ with $q = 1-1/p$ and the stochastic first-order oracles
  $\StochOrac{L}$ for this class. Then the radius $r$ of the largest
  $\ell_\infty$ ball inscribed within the $\Ball_q(1)$ scales as
  $\inrad = d^{-1/q}$.  By inspection of the lower bounds
  bounds~\eqref{EqnConvLip12} and~\eqref{EqnConvLipG2}, we see that
\begin{align}
 \label{EqnDualLB}
 \sup_{\phi \in \StochOrac{\Lval}} \bestacc_T(\Fconv,\Ball_q(1);
 \info) & = \begin{cases}
\Omega \left(\Lval \; \frac{d^{1/2-1/q}}{\sqrt{T}} \right) & \mbox{for
  $1 \leq p \leq 2$} \\ \Omega \left( \frac{\Lval}{T} \right) & \mbox{for
  $p \geq 2$.}
 \end{cases}
 \end{align}
As mentioned previously, the dimension-independent lower bound for the
case $p \geq 2$ was demonstrated in Chapter 5 of NY, and shown to be
optimal\footnote{There is an additional logarithmic factor in the
  upper bounds for $p = \Omega(\log d)$.} since it is achieved
using mirror descent with the prox-function $\|\cdot\|_q$.
For the case of $1 \leq p < 2$, the lower bounds are also
unimprovable, since they are again achieved (up to constant factors)
by stochastic gradient descent.  See Appendix~\ref{AppMirror} for
further details on these matching upper bounds.

\item[(b)] Let us now consider how our bounds can also make sharp
  predictions for non-dual geometries, using the special case
  $\ConvSet = \Ball_\infty(1)$.  For this choice, we have
  $\inrad(\Sset) = 1$, and hence Theorem~\ref{ThmConvex} implies that
  for all $\pval \in [1,2]$, the minimax oracle complexity is lower
  bounded as
\begin{align*}
\sup_{\phi \in \StochOrac{\Lval}} \bestacc_T(\Fconv,\Ball_\infty(1); \info) & =
\Omega \left(\Lval \; \sqrt{\frac{d}{T}} \right).
\end{align*}
This lower bound is sharp for all $\pval \in [1,2]$.  Indeed, for any
convex set $\ConvSet$, stochastic gradient descent achieves a matching
upper bound (see Section 5.2.4, p. 196 of NY~\cite{yudin83book}, as
well as Appendix~\ref{AppMirror} in this paper for further
discussion).

\item[(c)] As another example, suppose that $\ConvSet = \Ball_2(1)$.
  Observe that this $\ell_2$-norm unit ball satisfies the relation
  $\Ball_2(1) \supset \frac{1}{\sqrt{\dim}} \Ball_\infty(1)$, so that
  we have $\inrad(\Ball_2(1)) = 1/\sqrt{\dim}$.  Consequently, for
  this choice, the lower bound~\eqref{EqnConvLip12} takes the form
\begin{align*}
\sup_{\phi \in \StochOrac{\Lval}} \bestacc_T(\Fconv,\Ball_2(1); \info) & =
\Omega \left(\Lval \; \frac{1}{\sqrt{T}} \right),
\end{align*}
which is a dimension-independent lower bound. This lower bound for
$\Ball_2(1)$ is indeed tight for $\pval \in [1,2]$, and as before,
this rate is achieved by stochastic gradient
descent~\cite{yudin83book}.
\item[(d)] Turning to the case of $\lipexp > 2$, when $\ConvSet =
  \Ball_{\infty}(1)$, the lower bound~\eqref{EqnConvLipG2} can be
  achieved (up to constant factors) using mirror descent with the dual
  norm $\|\cdot\|_\lipexpdual$; for further discussion, we again refer
  the reader to Section 5.2.1, p.~190 of NY~\cite{yudin83book}, as
  well as to Appendix~\ref{AppMirror} of this paper. Also, even though
  this lower bound requires the oracle to have only bounded variance,
  our proof actually uses a stochastic oracle based on Bernoulli
  random variables, for which all moments exist.  Consequently, at
  least in general, our results show that there is no hope of
  achieving faster rates by restricting to oracles with bounds on
  higher-order moments.  This is an interesting contrast to the case
  of having \emph{less} than two moments, in which the rates are
  slower. For instance, as shown in Section 5.3.1 of
  NY~\cite{yudin83book}, suppose that the gradient estimates in a
  stochastic oracle satisfy the moment bound $\E\|\oracleg(x)\|_p^b
  \leq \sigma^2$ for some $b \in [1,2)$.  In this setting, the oracle
    complexity is lower bounded by $\Omega \big(T^{-(b-1)/b}\big)$.
    Since $T^{\frac{b-1}{b}} \ll T^{\frac{1}{2}}$ for all $b \in [1,
      2)$, there is a significant penalty in convergence rates for
      having less than two bounded moments.
\item[(e)] Even though the results have been stated in a first-order
  stochastic oracle model, they actually hold in a stronger sense.
  Let $\nabla^i f(x)$ denote the $i_{th}$-order derivative of $f$
  evaluated at $x$, when it exists. With this notation, our results
  apply to an oracle that responds with a random function $\hat{f}_t$
  such that
  \begin{equation*}
    \E [\hat{f}_t(x)] = \E[f(x)], \quad \mbox{and} \quad \E [ \nabla^i
      \hat{f}_t(x)] = \nabla^i f(x) \quad \mbox{for all $x \in \set$
      and $i$ such that $\nabla^i f(x)$ exists,}
  \end{equation*}
 along with appropriately bounded second moments of all the
 derivatives.  Consequently, higher-order gradient information cannot
 improve convergence rates in a worst-case setting. Indeed, the result
 continues to hold even for the significantly stronger oracle that
 responds with a random function that is a noisy realization of the
 true function. In this sense, our result is close in spirit to a
 statistical sample complexity lower bound.  Our proof technique is
 based on constructing a ``packing set'' of functions, and thus has
 some similarity to techniques used in statistical minimax analysis
 (e.g.,~\cite{Hasminskii78,Birge83,YanBar99,Yu97}) and learning theory
 (e.g.,~\cite{VapnikChe74, EhrenfeuchtHaKeVa88, Simon97}). A
 significant difference, as will be shown shortly, is that the metric
 of interest for optimization is very different than those typically
 studied in statistical minimax theory.
\end{enumerate}

\subsection{Oracle complexity for strongly convex Lipschitz functions}
We now turn to the statement of lower bounds over the class of
Lipschitz and strongly convex functions $\Fstrong$ from
Definition~\ref{DefnStrongConvex}.  In all these statements, we assume
that $\strongcon^2 \leq \frac{4\Lval\dim^{-1/\lipexp}}{\inrad}$, as is
required for the definition of $\Fstrong$ to be sensible.
\begin{theorem}
\label{ThmStrong}
Let $\ConvSet = \Ball_\infty(\inrad)$.  Then there exist universal
constants $c_1, c_2 > 0$ such that the minimax oracle complexity over
the class $\Fstrong(\ConvSet, \pval; \Lval, \strongcon)$ satisfies the
following lower bounds:
\begin{itemize}
\item[(a)] For $p = 1$, we have
\begin{align}
\label{EqnStrong12}
  \sup_{\info \in \StochOrac{\Lval}} \bestacc(\Fstrong,\info) & \geq
  \min\left\{ c_1\frac{\Lval^2}{\strongcon^2T}, \;
  c_2\Lval\inrad\sqrt{\frac{d}{T}}, \;
  \frac{\Lval^2}{1152\strongcon^2d}, \;
  \frac{\Lval\inrad}{144}\right\}.
\end{align}
\item[(b)] For $p > 2$, we have:
\begin{align}
\label{EqnStrongG2}
 \sup_{\info \in \StochOrac{\Lval}} \bestacc(\Fstrong,\info) & \geq
 \min \left( c_1\frac{L^2d^{1-2/\lipexp}}{\strongcon^2T}, \;
 c_2\frac{\Lval\inrad d^{1-1/\lipexp}}{\sqrt{T}}, \;
 \frac{\Lval^2d^{1-2/\lipexp}}{1152\strongcon^2}, \; \frac{\Lval\inrad
   d^{1-1/\lipexp}}{144} \right).
\end{align}
\end{itemize}
\end{theorem}

As with Theorem~\ref{ThmConvex}, these lower bounds are sharp. In
particular, for $S = \Ball_\infty(1)$, stochastic gradient descent
achieves the rate~\eqref{EqnStrong12} up to logarithmic
factors~\cite{Hazan2007logarithmic}, and closely related algorithms
proposed in very recent works~\cite{hazan10epochgd,JuditskyNe10} match
the lower bound exactly up to constant factors.  It should be noted
Theorem~\ref{ThmStrong} exhibits an interesting phase transition
between two regimes.  On one hand, suppose that the strong convexity
parameter $\strongcon^2$ is large: then as long as $T$ is sufficiently
large, the first term $\Omega(1/T)$ determines the minimax rate, which
corresponds to the fast rate possible under strong convexity.  In
contrast, if we consider a poorly conditioned objective with
$\strongcon \approx 0$, then the term involving $\Omega(1/\sqrt{T})$
is dominant, corresponding to the rate for a convex objective.  This
behavior is natural, since Theorem~\ref{ThmStrong} recovers (as a
special case) the convex result with $\strongcon = 0$.  However, it
should be noted that Theorem~\ref{ThmStrong} applies only to the set
$\Ball_{\infty}(\inrad)$, and not to arbitrary sets $\ConvSet$ like
Theorem~\ref{ThmConvex}.  Consequently, the generalization of
Theorem~\ref{ThmStrong} to arbitrary convex, compact sets remains an
interesting open question.


\subsection{Oracle complexity for convex Lipschitz functions with sparse optima}

Finally, we turn to the oracle complexity of optimization over the
class $\Fsparse$ from Definition~\ref{DefnSparseFunc}.
\begin{theorem}
\label{ThmSparse}
  Let $\Fsparse$ be the class of all convex functions that are
  $L$-Lipschitz with respect to the $\|\cdot\|_\infty$ norm and that have
  a $\kdim$-sparse optimizer. Let $\ConvSet \subset \real^\dim$ be a
  convex set with $\Ball_\infty(\inrad) \subseteq \ConvSet$. Then
  there exists a universal constant $c_0 > 0$ such that for all $\kdim
  \leq \lfloor \frac{\dim}{2} \rfloor$, we have
\begin{align}
\label{EqnSparseLower}
\sup_{\info \in \StochOractwoparam{\infty}{\Lval}}
\bestacc(\Fsparse,\info) & \geq \min \left(c_0\Lval\inrad\sqrt{
  \frac{\kdim^2 \log
    \frac{\dim}{\kdim}}{T}},\frac{\Lval\kdim\inrad}{432} \right).
\end{align}
\end{theorem}

\paragraph{Remark:} If $\kdim = \HackoRama$ for some
$\delta \in (0,1)$ (so that $\log \frac{\dim}{\kdim} = \Theta(\log
\pdim)$), then this bound is sharp up to constant factors.  In
particular, suppose that we use mirror descent based on the
$\|\cdot\|_{1+\varepsilon}$ norm with \mbox{$\varepsilon = 2\log
  d/(2\log d - 1)$.}  As we discuss in more detail in
Appendix~\ref{AppMirror}, it can be shown that this technique will
achieve a solution accurate to $\order(\sqrt{\frac{\kdim^2 \log
    \dim}{T}})$ within $T$ iterations; this achievable result matches
our lower bound~\eqref{EqnSparseLower} up to constant factors under
the assumed scaling $\kdim = \HackoRama$ .  To the best of our
knowledge, Theorem~\ref{ThmSparse} provides the first tight lower
bound on the oracle complexity of sparse optimization.


\comment{
\subsection{A general result}
Armed with the greater understanding from these proofs, we can now
state a general result for any function class $\F$.  The proof is
similar to that of earlier results.

\begin{theorem}
For any function class $\F \subseteq \Fconv$, suppose a given base set
of functions $\funclassset$ yields the subclass $\G(\err)$ and the
measure $\errfun$ as defined in \eqref{DefnPackingSize}, and this
function is monotone increasing. Then there exists a universal
constant $c > 0$ such that $\sup_{\info \in \O}
\bestacc(\Fstrong,\info) \geq c \,\errfun
\big(\sqrt{\frac{\log|\G(\err)|}{T}} \big)$.
\end{theorem}
}


\section{Proofs of results}
\label{sec:proofs}

We now turn to the proofs of our main results.  We begin in
Section~\ref{SecBasic} by outlining the framework and establishing
some basic results on which our proofs are based.
Sections~\ref{SecProofConvex} through~\ref{SecProofSparse} are devoted
to the proofs of Theorems~\ref{ThmConvex} through~\ref{ThmSparse}
respectively.

\subsection{Framework and basic results}
\label{SecBasic}

We begin by establishing a basic set of results that are exploited in
the proofs of the main results.  At a high-level, our main idea is to
show that the problem of convex optimization is at least as hard as
estimating the parameters of Bernoulli variables---that is, the biases
of $\dim$ independent coins.  In order to perform this embedding, for
a given error tolerance $\epsilon$, we start with an appropriately
chosen subset of the vertices of a $\dim$-dimensional hypercube, each
of which corresponds to some values of the $\dim$ Bernoulli
parameters. For a given function class, we then construct a
``difficult'' subclass of functions that are indexed by these vertices
of the hypercube.  We then show that being able to optimize any
function in this subclass to $\epsilon$-accuracy requires identifying
the hypercube vertex.  This is a multiway hypothesis test based on the
observations provided by $T$ queries to the stochastic oracle, and we
apply Fano's inequality~\cite{Cover} or Le Cam's
bound~\cite{LeCam1973,Yu97} to lower bound the probability of error.
In the remainder of this section, we provide more detail on each of
steps involved in this embedding.

\nocite{Hasminskii78,Birge83,Yu97}

\subsubsection{Constructing a difficult subclass of functions}

Our first step is to construct a subclass of functions $\G \subseteq
\FunSet$ that we use to derive lower bounds.  Any such subclass is
parametrized by a subset $\packset \subseteq \{-1, +1\}^d$ of the
hypercube, chosen as follows.  Recalling that $\Hamm$ denotes the
Hamming metric, we let $\packset = \{\alpha^1, \ldots, \alpha^M \}$ be
a subset of the vertices of the hypercube such that
\begin{align}
\label{EqnPackingDistance}
\Hamm(\alpha^j, \alpha^k) \geq \frac{\dim}{4} \quad \mbox{for all $j
  \neq k$,}
\end{align}
meaning that $\packset$ is a $\frac{\dim}{4}$-packing in the Hamming
norm.  It is a classical fact (e.g.,~\cite{Matousek}) that one can
construct such a set with cardinality $|\packset| \geq
(2/\sqrt{e})^{d/2}$.

Now let $\funclassset = \{\fp_i,\fn_i, \, i = 1,\hdots, \pdim\}$
denote some base set of $2 d$ functions defined on the convex set
$\ConvSet$, to be chosen appropriately depending on the problem at hand.
For a given tolerance $\err \in (0, \frac{1}{4}]$, we define, for each
vertex $\parama \in \packset$, the function
\begin{align}
\label{EqnGClassFun}
  g_\parama(x) & \defn \frac{c}{\pdim} \sum_{i =1}^\pdim \big \{ (1/2 +
  \parama_i \err ) \fp_i(x) + (1/2 - \parama_i \err) \,
  \fn_i(x)\big\}.
\end{align}
Depending on the result to be proven, our choice of the base functions
$\{\fp_i, \fn_i\}$ and the pre-factor $c$ will ensure that each
$g_\parama$ satisfies the appropriate Lipschitz and/or strong
convexity properties over $\ConvSet$.  Moreover, we will ensure that
that all minimizers $x_\parama$ of each $g_\parama$ are contained
within $\ConvSet$.

Based on these functions and the packing set $\packset$, we define the
function class
\begin{align}
\G(\err) & \defn \big \{ g_\parama, \; \parama \in \packset \big \}.
\end{align}
Note that $\G(\err)$ contains a total of $|\packset|$ functions by
construction, and as mentioned previously, our choices of the base
functions etc. will ensure that $\G(\err) \subseteq \FunSet$.  We
demonstrate specific choices of the class $\G(\err)$ in the proofs of
Theorems~\ref{ThmConvex} through~\ref{ThmSparse} to follow.

\subsubsection{Optimizing well is equivalent to function identification} 

We now claim that if a method can optimize over the subclass
$\G(\err)$ up to a certain tolerance, then it must be capable of
identifying which function $g_\parama \in \G(\err)$ was chosen.  We
first require a measure for the \emph{closeness} of functions in terms
of their behavior near each others' minima.  Recall that we use
$\elemstar_f \in \real^\pdim$ to denote a minimizing point of the
function $f$.  Given a convex set $S \subseteq \mathbb{R}^d$ and two
functions $f,g$, we define
\begin{align}
\label{EqnDefnDisc}
 \disc(f,g) & \defn \inf_{x \in \Sset} \big[f(x) + g(x) - f(\elemstar_f) -
 g(\elemstar_g) \big].
\end{align}
\begin{figure}[h]
\begin{center}
\psfrag{*xf*}{$\elemstar_f$} \psfrag{*xg*}{$\elemstar_g$}
\psfrag{*fs*}{$f(\elemstar_f)$} \psfrag{*gs*}{$g(\elemstar_g)$}
\psfrag{*fpgs*}{$\inf_{x \in S} \big \{ f(x) + g(x) \}$}
\widgraph{.5\textwidth}{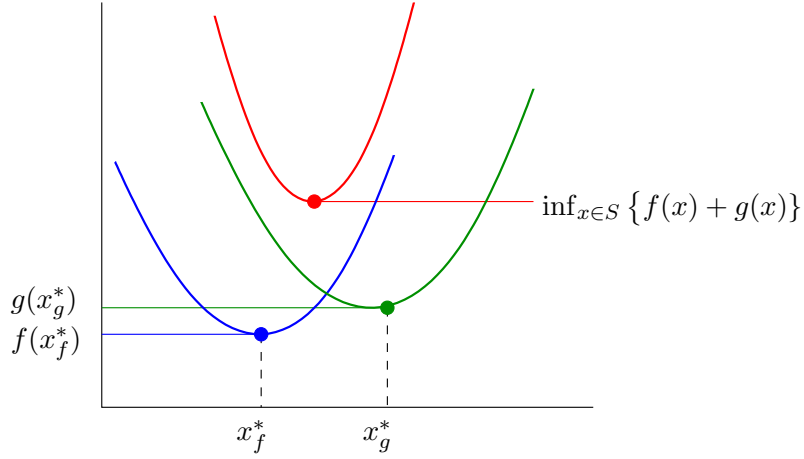}
\end{center}
\caption{Illustration of the discrepancy function $\disc(f,g)$.  The
  functions $f$ and $g$ achieve their minimum values $f(\elemstar_f)$
  and $g(\elemstar_g)$ at the points $\elemstar_f$ and $\elemstar_g$
  respectively.}
\end{figure}
This discrepancy measure is non-negative, symmetric in its arguments,
and satisfies $\disc(f,g) = 0$ if and only if $\elemstar_f =
\elemstar_g$, so that we may refer to it as a premetric.  (It does not
satisfy the triangle inequality nor the condition that $\disc(f,g) =
0$ if and only if $f = g$, both of which are required for $\disc$ to
be a metric.)

Given the subclass $\G(\err)$, we quantify how densely it is packed
with respect to the premetric $\disc$ using the quantity
\begin{align}
\label{DefnPackingSize}
\errfun(\G(\err)) & \defn \min_{\parama \neq \paramb \in \packset}
\disc(g_\parama, g_\paramb).
\end{align}
We denote this quantity by $\errfun(\err)$ when the class $\G$ is
clear from the context.  We now state a simple result that
demonstrates the utility of maintaining a separation under $\disc$
among functions in $\G(\err)$.
\begin{lemma}\label{lemma:redndet}
For any $\widetilde{\elem} \in \ConvSet$, there can be at most one
function $g_\parama \in \G(\err)$ such that
\begin{align}
\label{EqnUnique} 
g_\parama(\widetilde{\elem}) - \inf_{x \in \ConvSet} g_\parama(x) & \leq
\frac{\errfun(\err)}{3}.
\end{align}
\end{lemma}

\noindent Thus, if we have an element $\widetilde{\elem} \in \ConvSet$
that approximately minimizes one function in the set $\G(\err)$ up to
tolerance $\errfun(\err)$, then it cannot approximately minimize any
other function in the set.
\begin{proof}
 For a given $\widetilde{\elem} \in \ConvSet$,
suppose that there exists an $\alpha \in \packset$ such that
$g_\parama(\widetilde{\elem}) - g_\parama(\elemstar_\parama) \leq
\frac{\errfun(\err)}{3}$. From the definition of $\errfun(\err)$ in
\eqref{DefnPackingSize}, for any $\beta \in \packset,\; \beta \neq
\alpha$, we have
\begin{eqnarray*}
\errfun(\err) & \leq & g_\parama(\widetilde{\elem}) - \inf_{x \in
  \ConvSet}g_\parama(x) + g_\beta(\widetilde{\elem}) - \inf_{x \in
  \ConvSet}g_\beta(x) \; \leq \; \frac{\errfun(\err)}{3} +
g_\beta(\widetilde{\elem}) - \inf_{x \in \ConvSet}g_\beta(x).
\end{eqnarray*}	
Re-arranging yields the inequality $g_{\beta}(\widetilde{\elem}) -
g_\beta(\elemstar_\beta) \ge \frac{2}{3} \, \errfun(\err)$, from which
the claim~\eqref{EqnUnique} follows.

\end{proof}
%


Suppose that for some fixed but unknown function $g_{\paramtrue} \in
\G(\err)$, some method $\meth_T$ is allowed to make $T$ queries to an
oracle with information function $\info(\cdot \,; \,g_{\paramtrue})$,
thereby obtaining the information sequence 
\begin{align*}
\Data{g_\paramtrue} & \defn \{ \info(x_t; g_\paramtrue), t = 1, 2,
\ldots, T \}.
\end{align*}
Our next lemma shows that if the method $\meth_T$ achieves a low
minimax error over the class $\G(\err)$, then one can use its output
to construct a hypothesis test that returns the true parameter
$\paramtrue$ at least $2/3$ of the time.  (In this statement, we
recall the definition~\eqref{EqnDefnMinimaxError} of the minimax error
in optimization.)
\begin{lemma}\label{lemma:redn}
Suppose that based on the data $\Data{g_\paramtrue}$, there exists a
 method $\meth_T$ that achieves a minimax error satisfying
\begin{equation}
\label{EqnMiniBound}
\E\big[\acc_T(\meth_T, \G(\err), \ConvSet, \info)\big] \leq
\frac{\errfun(\err)}{9}.
\end{equation}
Based on such a method $\meth_T$, one can construct a hypothesis test
$\alhat: \Data{g_\paramtrue} \rightarrow \packset$ such that $\max
\limits_{\paramtrue \in \packset} \P_\truth[\alhat \ne \paramtrue]
\leq \frac{1}{3}$.
\end{lemma}
\begin{proof}
Given a method $\meth_T$ that satisfies the
bound~\eqref{EqnMiniBound}, we construct an estimator $\esta{\meth_T}$
of the true vertex $\paramtrue$ as follows.  If there exists some
$\alpha \in \packset$ such that $g_\parama(x_T) - g_\parama(x_\parama)
\leq \frac{\errfun(\err)}{3}$ then we set $\esta{\meth_T}$ equal to
$\alpha$. If no such $\alpha$ exists, then we choose $\esta{\meth_T}$
uniformly at random from $\packset$. From Lemma~\ref{lemma:redndet},
there can exist only one such $\alpha \in \packset$ that satisfies
this inequality.  Consequently, using Markov's inequality, we have
$\P_{\truth}[\esta{\meth_T} \ne \paramtrue] \leq \P_{\truth}
\big[\acc_T(\meth_T, g_{\paramtrue}, \ConvSet, \info) \geq
\errfun(\err)/3 \big] \; \leq \; \frac{1}{3}$.  Maximizing over
$\paramtrue$ completes the proof.
\end{proof}

\noindent We have thus shown that having a low minimax optimization
error over $\G(\err)$ implies that the vertex $\paramtrue \in
\packset$ can be identified most of the time.


\subsubsection{Oracle answers and coin tosses} 

We now describe stochastic first order oracles $\phi$ for which the
samples $\Data{g_\parama}$ can be related to coin tosses.  In
particular, we associate a coin with each dimension $i \in \{1, 2,
\ldots, \pdim\}$, and consider the set of coin bias vectors lying in
the set
\begin{eqnarray}\label{EqnBiasSet}
\coins(\err) = \big\{ (1/2 + \parama_1\err,\dots, 1/2 +  
\parama_d \err ) \, \mid \, \parama \in \packset \big\},
\end{eqnarray}
Given a particular function $g_\parama \in \G(\delta)$---or
equivalently, vertex $\parama \in \packset$---we consider two
different types of stochastic first-order oracles $\phi$, defined as
follows: \\

\framebox[0.95\textwidth]{
\parbox{0.9\textwidth}{
\paragraph{Oracle A: 1-dimensional unbiased gradients}
\begin{enumerate}
  \item[(a)] Pick an index $i \in \{1,\dots,d \}$ uniformly at random.
  \item[(b)] Draw $b_{i} \in \{0,1\}$ according to a Bernoulli
  distribution with parameter $1/2 + \parama_{i} \err$.

 \item[(c)] For the given input $x \in \ConvSet$, return the value
$\oracfuna{\parama}(x)$ and a sub-gradient \mbox{$\oracsuba{\parama}
(x) \in \partial \oracfuna{\parama}(x)$} of the function
  \begin{equation*}
    \oracfuna{\parama} \; \defn \, c \big[b_{i}\fp_{i} + (1-b_{i})
    \fn_{i} \big].
  \end{equation*}
\end{enumerate}
} } \\

By construction, the function value and gradients returned by Oracle A
are unbiased estimates of those of $g_\alpha$.  In particular, since
each co-ordinate $i$ is chosen with probability $1/\dim$, we have
\begin{align*}
\Exs \big[ \oracfuna{\parama}(x)\big] & =\frac{c}{d} \sum_{i=1}^d
\big[\Exs [b_{i}]\fp_{i}(x) + \Exs[1-b_{i}] \fn_i(x) \big] \; = \;
g_\alpha(x),
\end{align*}
with a similar relation for the gradient.  Furthermore, as long as the
base functions $\fp_i$ and $\fn_i$ have gradients bounded by $1$, we
have $\Exs[\|\oracsuba{\parama}(x)\|_\pval] \leq c$ for all $\pval \in
[1, \infty]$.  \\

Parts of proofs are based on an oracle which responds with function
values and gradients that are \emph{$\dim$-dimensional} in nature.\\

\framebox[0.95\textwidth]{
\parbox{0.9\textwidth}{
\paragraph{Oracle B: $\dim$-dimensional unbiased gradients}
\begin{enumerate}
    \item[(a)] For \mbox{$i = 1,\dots,\dim$,} draw $b_i \in \{0,1\}$
according to a Bernoulli distribution with \mbox{parameter} $1/2 +
\parama_{i} \err$.
\item[(b)] For the given input $x \in \ConvSet$, return the value
$\oracfunb{\parama}(x)$ and a sub-gradient
\mbox{$\oracsubb{\parama}(x) \in \partial \oracfunb{\parama}(x)$} of
the function
\begin{equation*}
\oracfunb{\parama} \; \defn \frac{c}{\dim}\sum_{i=1}^\dim \big[b_i
 \fp_i + (1-b_i) \fn_i \big].
\end{equation*}
\end{enumerate}
}} \\

As with Oracle A, this oracle returns unbiased estimates of the
function values and gradients.  We frequently work with functions
$\fp_i, \fn_i$ that depend only on the $i^{th}$ coordinate $\x{i}$. In
such cases, under the assumptions $|\frac{\partial \fp_i}{\x{i}}| \leq
1$ and $|\frac{\partial \fn_i}{\x{i}}| \leq 1$, we have
\begin{align}
  \label{eqn:fulldnorm}
\|\oracsubb{\parama}(x)\|_p^2 &= \frac{c^2}{d^2} \left(
\sum_{i=1}^\dim \left| b_i \frac{\partial\fp_i(x)}{\partial \x{i}} +
(1-b_i) \frac{\partial\fn_i(x)}{\partial \x{i}} \right|^p \right )^{2/p}
\; \leq \; c^2 d^{2/p-2}.
\end{align}
In our later uses of Oracles A and B, we choose the pre-factor $c$
appropriately so as to produce the desired Lipschitz constants.

\subsubsection{Lower bounds on coin-tossing}   

Finally, we use information-theoretic methods to lower bound the
probability of correctly estimating the true parameter $\paramtrue \in
\packset$ in our model.  At each round of either Oracle A or Oracle B,
we can consider a set of $\dim$ coin tosses, with an associated vector
$\theta^* = (\frac{1}{2} + \paramtrue_1 \err, \ldots, \frac{1}{2} +
\paramtrue_\dim \err)$ of parameters.  At any round, the output of
Oracle A can (at most) reveal the instantiation $b_i \in \{0,1\}$ of a
randomly chosen index, whereas Oracle B can at most reveal the entire
vector $(b_1, b_2, \ldots, b_\dim)$.  Our goal is to lower bound the
probability of estimating the true parameter $\paramtrue$, based on a
sequence of length $T$. As noted previously in remarks following
Theorem~\ref{ThmConvex}, this part of our proof exploits classical
techniques from statistical minimax theory, including the use of
Fano's inequality (e.g.,~\cite{Hasminskii78,Birge83,YanBar99,Yu97})
and Le Cam's bound (e.g.,~\cite{LeCam1973,Yu97}).

\begin{lemma}
\label{LemFano}
Suppose that the Bernoulli parameter vector $\Rparam$ is chosen
uniformly at random from the packing set $\packset$, and suppose that
the outcome of $\ell \leq \dim$ coins chosen uniformly at random is
revealed at each round $t = 1, \ldots, T$.  Then for any $\err \in (0,
1/4]$, any hypothesis test $\widehat{\rparam}$ satisfies
\begin{align}
\label{EqnFanoLower}
\P [\widehat{\rparam} \neq \Rparam] & \geq 1 - \frac{16 \ell T \err^2
  + \log 2}{\frac{d}{2}\log(2/\sqrt{e})},
\end{align}
where the probability is taken over both randomness in the oracle and
the choice of $\Rparam$.
\end{lemma}
\noindent Note that we will apply the lower bound~\eqref{EqnFanoLower}
with $\ell = 1$ in the case of Oracle A, and $\ell = \dim$ in the case
of Oracle B.
\begin{proof}
For each time $t = 1, 2, \ldots, T$, let $\Mysub_t$ denote the
randomly chosen subset of size $\ell$, $X_{t,i}$ be the outcome of
oracle's coin toss at time $t$ for coordinate $i$ and let $Y_t \in
\{-1, 0, 1\}^\dim$ be a random vector with entries
\begin{align*}
Y_{t,i} & = \begin{cases} X_{t,i} & \mbox{if $i \in \Mysub_t$, and} \\
-1 & \mbox{if $i \notin \Mysub_t$.}
	    \end{cases}
\end{align*}
By Fano's inequality~\cite{Cover}, we have the lower bound
\begin{align*}
\P [\widehat{\rparam} \neq \Rparam] & \geq 1 - \frac{ I(\{(\Mysub_t,
  Y_t\}_{t=1}^T; \Rparam) + \log 2}{\log |\packset|},
\end{align*}
where $I(\{(\Mysub_t, Y_t\}_{t=1}^T; \Rparam)$ denotes the mutual
information between the sequence $\{(\Mysub_t, Y_t)\}_{t=1}^T$ and the
random parameter vector $\Rparam$.  As discussed earlier, we are
guaranteed that $\log |\packset| \geq \frac{d}{2}\log(2/\sqrt{e})$.
Consequently, in order to prove the lower bound~\eqref{EqnFanoLower},
it suffices to establish the upper bound $I(\{\Mysub_t, Y_t\}_{t=1}^T;
\Rparam) \: \leq 16 T \: \ell \; \err^2$.

 By the independent and identically distributed nature of the sampling
 model, we have
\begin{align*}
I(((\Mysub_1, Y_1), \ldots, (\Mysub_T, Y_T)); \Rparam) & =
\sum_{t=1}^T I((\Mysub_t,Y_t); \Rparam) \; = \; T \; I((\Mysub_1,
Y_1); \Rparam),
\end{align*}
so that it suffices to upper bound the mutual information for a single
round.  To simplify notation, from here onwards we write $(Y, \Mysub)$
to mean the pair $(Y_1, \Mysub_1)$.  With this notation, the remainder
of our proof is devoted to establishing that $I(Y; \Mysub) \leq 16 \:
\ell \: \err^2$,

 By chain rule for mutual information~\cite{Cover}, we have
\begin{align}
\label{EqnChain}
I((\Mysub, Y); \Rparam) & = I(Y; \Rparam \mid \Mysub) + I(\Rparam;
\Mysub).
\end{align}
Since the subset $\Mysub$ is chosen independently of $\Rparam$, we
have $I(\Rparam; \Mysub) = 0$, and so it suffices to upper bound the
first term.  By definition of conditional mutual
information~\cite{Cover}, we have
\begin{align*}
I(Y; \Rparam \, \mid \, \Mysub) & = \Exs_{\Mysub} \big[ \KL{\mprob_{Y
      \mid \Rparam, \Mysub}}{\; \mprob_{Y \mid \Mysub}} \big]
\end{align*}
Since $\rparam$ has a uniform distribution over $\packset$, we have
$\mprob_{Y \mid \Mysub} = \frac{1}{|\packset|} \sum_{\rparam \in
  \packset} \mprob_{Y \mid \rparam, \Mysub}$, and convexity of the
Kullback-Leibler (KL) divergence yields the upper bound
\begin{align}
\label{EqnConvexBound}
\KL{\mprob_{Y \mid \Rparam, \Mysub}}{\; \mprob_{Y \mid \Mysub}} & \leq
\frac{1}{|\packset|} \sum_{\rparam \in \packset} \kull{\mprob_{Y \mid
    \Rparam, \Mysub}}{\; \mprob_{Y \mid \rparam, \Mysub}}.
\end{align}

Now for any pair $\Rparam, \rparam \in \packset$, the KL divergence
$\kull{\mprob_{Y \mid \Rparam, \Mysub}}{\; \mprob_{Y \mid \rparam,
    \Mysub}}$ can be at most the KL divergence between $\ell$
independent pairs of Bernoulli variates with parameters $\half + \err$
and $\half - \err$.  Letting $D(\err)$ denote the Kullback-Leibler
divergence between a single pair of Bernoulli variables with
parameters $\half + \err$ and $\half - \err$, a little calculation
yields
\begin{align*}
D(\err) & = \left(\half + \err\right) \log \frac{\half + \err}{\half -
  \err} +
\left(\half - \err\right) \log \frac{\half - \err}{\half + \err} \\
& = 2 \err \log \left(1 + \frac{4 \err}{1-2\err} \right) \\
& \leq \frac{8 \err^2}{1-2\err}.
\end{align*}
Consequently, as long as $\err \leq 1/4$, we have $D(\err) \leq 16
\err^2$.  Returning to the bound~\eqref{EqnConvexBound}, we conclude
that $\KL{\mprob_{Y \mid \Rparam, \Mysub}}{\; \mprob_{Y \mid \Mysub}}
\leq 16 \: \ell \: \err^2$.  Taking averages over $\Mysub$, we obtain
the bound \mbox{$I(Y; \Rparam \, \mid \, \Mysub) \leq 16 \; \ell \;
  \err^2$,} and applying the decomposition~\eqref{EqnChain} yields
$I((\Mysub, Y); \Rparam) \leq 16 \: \ell \: \err^2$, thereby
completing the proof.
\end{proof}


The reader might have observed that Fano's inequality yields a
non-trivial lower bound only when $|\packset|$ is large enough. Since
$|\packset|$ depends on the dimension $d$ for our construction, we can
apply the Fano lower bound only for $d$ large enough.  Smaller values
of $d$ can be lower bounded by reduction to the case $d = 1$; here we
state a simple lower bound for estimating the bias of a single coin,
which is a straightforward application of Le Cam's bounding
technique~\cite{LeCam1973,Yu97}.  In this special case, we have
$\packset = \left\{1/2 + \err, 1/2 - \err\right\}$, and we recall that
the estimator $\alhat(\meth_T)$ takes values in $\packset$.

\begin{lemma}
Given a sample size $T \geq 1$ and a parameter $\parama^* \in
\packset$, let $\{X_1, \ldots, X_T\}$ be $T$ i.i.d Bernoulli variables
with parameter $\parama^*$. Let $\alhat$ be any test function based on
these samples and returning an element of $\packset$.  Then for any
$\err \in (0,1/4]$, we have the lower bound
  \begin{align*}
    \sup_{\parama^* \in \{\half+\err,\half - \err\}}
    \P_{\parama^*}[\widehat{\rparam} \ne \parama^*] \geq 1 -\sqrt{8T\err^2}.
  \end{align*}
  \label{lemma:onecoin}
\end{lemma}

\begin{proof}
We observe first that for $\alhat \in \packset$,
$\E_{\parama^*}[|\alhat - \parama^*|] = 2\err\P_{\parama^*} [\alhat
  \ne \parama^*]$, so that it suffices to lower bound the expected
error. To ease notation, let $\qprob_1$ and $\qprob_{-1}$ denote the
probability distributions indexed by $\parama = \frac{1}{2} + \err$
and $\parama = \frac{1}{2} - \err$ respectively. By Lemma 1 of
Yu~\cite{Yu97}, we have
  \begin{align*}
    \sup_{\parama^* \in \packset} \E_{\parama^*} [|\alhat -
      \parama^*|] \geq 2 \err \Big \{ 1 - \|\qprob_1 -
\qprob_{-1}\|_1/2 \Big \}.
  \end{align*}
where we use the fact that $|(1/2 + \err) - (1/2 - \err)| = 2\err$.
Thus, we need to upper bound the total variation distance $\|\qprob_1
- \qprob_{-1}\|_1$.  From Pinkser's inequality~\cite{Cover}, we have
\begin{align*}
\|\qprob_1 - \qprob_{-1}\|_1 & \leq \sqrt{2
  \KL{\qprob_1}{\qprob_{-1}}} \; \stackrel{(i)}{\leq} \sqrt{32T\err^2},
\end{align*}
where inequality (i) follows from the calculation following
Equation~\ref{EqnConvexBound} (see proof of Lemma~\ref{LemFano}), and
uses our assumption that $\err \in (0,1/4]$.  Putting together the
  pieces, we obtain a lower bound on the probability of error
  \begin{align*}
    \sup_{\parama^* \in \packset} \P[\alhat \ne \parama^*] \;=\;
    \sup_{\parama^* \in \packset} \frac{\E|\alhat - \parama^*|}{2\err}
    \; \geq \; 1 - \sqrt{8T\err^2},
  \end{align*}  
as claimed.
\end{proof}

\noindent Equipped with these tools, we are now prepared to prove our
main results.

\subsection{Proof of Theorem~\ref{ThmConvex}}
\label{SecProofConvex}

We begin with oracle complexity for bounded Lipschitz functions, as
stated in Theorem~\ref{ThmConvex}.  We first prove the result for the
set $\ConvSet = \Ball_\infty(\frac{1}{2})$.

\paragraph{Part (a)---Proof for $\pval \in [1,2]$:}
Consider Oracle A that returns the quantities $(\oracfuna{\parama}(x),
\oracsuba{\parama}(x))$.  By definition of the oracle, each round
reveals only at most one coin flip, meaning that we can apply
Lemma~\ref{LemFano} with $\ell = 1$, thereby obtaining the lower bound
\begin{equation}
  \P[\esta{\meth_T} \ne \parama] \geq 1 - 2\frac{16T \err^2 + \log
    2}{\dim \log(2/\sqrt{e})}.
  \label{eqn:lowerbound}
\end{equation}

We now seek an upper bound $\P[\esta{\meth_T} \not= \parama]$ using
Lemma~\ref{lemma:redn}.  In order to do so, we need to specify the
base functions $(\fp_i, \fn_i)$ involved.  For $i=1, \ldots, \dim$, we
define
\begin{equation}
\label{EqnPrecedingFPFN}
  \fp_i(x) \defn \left|\x{i} + \half \right|, \quad \mbox{and} ~\quad
  \fn_i(x) \defn \left | \x{i} - \half \right|.
\end{equation}
Given that $\ConvSet = \Ball_\infty(\frac{1}{2})$, we see that the
minimizers of $g_\parama$ are contained in $S$. Also, both the
functions are 1-Lipschitz in the $\ell_1$-norm.  By the
construction~\eqref{EqnGClassFun}, we are guaranteed that for any
subgradient of $g_\parama$, we have
\begin{equation*}
\|\oracsuba{\parama}(x)\|_\pval \leq 2 c \qquad \mbox{for all $\pval
  \geq 1$.}
\end{equation*}
Therefore, in order to ensure that $g_\parama$ is $L$-Lipschitz in the
dual $\ell_\qpar$-norm, it suffices to set $c = L/2$.

 Let us now lower bound the discrepancy
 function~\eqref{EqnDefnDisc}. We first observe that each function
 $g_\parama$ is minimized over the set
 $\Ball_\infty\big(\frac{1}{2}\big)$ at the vector $x_\alpha \defn
 -\alpha /2$, at which point it achieves its minimum value
\begin{equation*}
  \min_{x \in \Ball_\infty (\frac{1}{2})} g_\parama(x) = \frac{c}{2}
  - c \delta.
\end{equation*}
Furthermore, we note that for any $\parama \ne \paramb$, we have
\begin{align*}
  g_\parama(x) + g_\paramb(x) &= \frac{c}{d} \sum_{i=1}^d \left
  [\left(\half + \parama_i \err+ \half + \paramb_i\err\right)\fp_i(x) +
    \left(\half - \parama_i \err+ \half - \paramb_i\err\right) \fn_i(x) \right] \\
& =   \frac{c}{d}\sum_{i=1}^d\left[\left(1 + \parama_i \err+
    \paramb_i \err \right)\fp_i(x) + \left(1 - \parama_i\err -
    \paramb_i\err\right)\fn_i(x)\right] \\ 
& = \frac{c}{d}\sum_{i=1}^d\left [ \left(\fp_i(x) + \fn_i(x)\right)
 \ind(\parama_i \ne \paramb_i) + \left((1 + 2\parama_i\err)\fp_i(x) +
  (1-2\parama_i\err)\fn_i(x)\right) \ind(\parama_i = \paramb_i) \right].
\end{align*}
When $\parama_i = \paramb_i$ then $x_\parama(i) = x_\paramb(i) =
-\parama_i/2$, so that this co-ordinate does not make a contribution
to the discrepancy function $\disc(g_\parama, g_\paramb)$. On the
other hand, when $\parama_i \ne \paramb_i$, we have
\begin{equation*}
  \fp_i(x) + \fn_i(x) = \left|x(i) + \half\right| + \left|x(i) -
  \half\right| \geq 1 \quad \mbox{for all $x \in \real$.}
\end{equation*}
Consequently, any such co-ordinate yields a contribution of $2 c \err/d$
to the discrepancy. Recalling our packing
set~\eqref{EqnPackingDistance} with $\dim/4$ separation in Hamming
norm, we conclude that for any distinct $\parama \neq \paramb$ within
our packing set,
\begin{equation*}
\disc(g_\parama,g_\paramb) = \frac{2c\err}{\dim} \;
\Hamm(\parama,\paramb) \geq \frac{c\err}{2},
\end{equation*}
so that by definition of $\errfun$, we have established the lower
bound $\errfun(\err) \geq \frac{c\err}{2}$.  

Setting the target error $\acc \defn \frac{c \err}{18}$, we observe
that this choice ensures that $\acc <
\frac{\errfun(\err)}{9}$. Recalling the requirement $\err < 1/4$, we
have $\acc < c/72$. In this regime, we may apply
Lemma~\ref{lemma:redn} to obtain the upper bound
$\P_\truth[\esta{\meth_T} \neq \parama] \leq \frac{1}{3}$.  Combining
this upper bound with the lower bound~\eqref{eqn:lowerbound} yields
the inequality
\begin{align*}
\frac{1}{3} & \geq 1 - 2\frac{16T \err^2 + \log 2}{\dim
  \log(2/\sqrt{e})}.
\end{align*}
Recalling that $c = \frac{L}{2}$, making the substitution $\err =
\frac{18 \acc}{c} = \frac{36 \acc}{L}$, and performing some algebra
yields
\begin{align*}
T & \geq c_0
\frac{L^2}{\epsilon^2}\left(\frac{d}{3}\log\left(\frac{2}{\sqrt{e}}\right)
- \log 2\right) \geq c_1\, \frac{L^2\dim}{\acc^2} \quad \mbox{ for all
  $\dim \geq 11$ and for all}~\epsilon \leq \frac{L}{144},
\end{align*}
where $c_0$ and $c_1$ are universal constants. Combined with
Theorem~5.3.1 of NY~\cite{yudin83book} (or by using the lower bound of
Lemma~\ref{lemma:onecoin} instead of Lemma~\ref{LemFano}), we conclude
that this lower bound holds for all \mbox{dimensions $\dim$.}


\paragraph{Part (b)---Proof for $\pval > 2$:}
The preceding proof based on Oracle A is also valid for $\pval > 2$,
but yields a relatively weak result.  Here we show how the use of
Oracle B yields the stronger claim stated in
Theorem~\ref{ThmConvex}(b).  When using this oracle, all $\dim$ coin
tosses at each round are revealed, so that Lemma~\ref{LemFano} with
$\ell=\dim$ yields the lower bound
\begin{equation}
\label{eqn:lowerboundfulld}
\P[\esta{\meth_T} \ne \parama] \geq 1 - 2\frac{16 \, T \, \dim \,
  \err^2 + \log 2}{\dim \log(2/\sqrt{e})}.
\end{equation}

We now seek an upper bound on $\P[\esta{\meth_T} \neq \parama]$.  As
before, we use the set $\ConvSet = \Ball_\infty(\frac{1}{2})$, and the
previous definitions~\eqref{EqnPrecedingFPFN} of $\fp_i(x)$ and
$\fn_i(x)$.  From our earlier analysis (in particular,
equation~\eqref{eqn:fulldnorm}), the quantity
$\|\oracsubb{\parama}(x)\|_p$ is at most $c\dim^{1/\lipexp-1}$, so that
setting $c = L \dim^{1-1/\lipexp}$ yields functions that are Lipschitz
with parameter $L$.

As before, for any distinct pair $\parama, \paramb \in \packset$, we
have the lower bound
\begin{align*}
\disc(g_\parama,g_\paramb) & = \frac{2c\err}{\dim} \;
\Hamm(\parama,\paramb) \; \geq \; \frac{c\err}{2},
\end{align*}
so that $\errfun(\err) \geq \frac{c \err}{2}$.  Consequently, if we
set the target error $\acc \defn \frac{c \err}{18}$, then we are
guaranteed that $\acc < \frac{\errfun(\err)}{9}$, as is required for
applying Lemma~\ref{lemma:redn}.  Application of this lemma yields the
upper bound $\P_\truth[\esta{\meth_T} \ne \parama] \leq \frac{1}{3}$.
Combined with the lower bound~\eqref{eqn:lowerboundfulld}, we obtain
the inequality
\begin{align*}
\frac{1}{3} & \geq 1 - 2\frac{16\, \dim \: T \: \err^2 + \log 2}{\dim
\log(2/\sqrt{e})}.
\end{align*}
Substituting $\err = 18\acc/c$ yields the scaling $\acc \geq c_0\,
\frac{c}{\sqrt{T}}$ for all $\dim \geq 11$, $\acc \leq c/72$ and a
universal constant $c_0$.  Recalling that $c = L \dim^{1-1/\lipexp}$,
we obtain the bound~\eqref{EqnConvLipG2}.  Combining this with
Theorem~5.3.1 of NY~\cite{yudin83book} (or by using the lower bound of
Lemma~\ref{lemma:onecoin} instead of Lemma~\ref{LemFano}) gives the
claim for all dimensions.

\vspace*{.2in}

We have thus completed the proof of Theorem~\ref{ThmConvex} in the
special case $\ConvSet = \Ball_\infty(\frac{1}{2})$.  In order to
prove the general claims, which scale with $\inrad$ when
$B_\infty(\inrad)\subseteq \ConvSet$, we note that our preceding proof
required only that $\ConvSet \supseteq \Ball_\infty(\frac{1}{2})$ so
that the minimizing points $x_\parama = -\parama/2 \in \ConvSet$ for
all $\parama$ (in particular, the Lipschitz constant of $g_\parama$
does not depend on $\ConvSet$ for our construction). In the general
case, we define our base functions to be
\begin{align*}
  \fp_i(x) = \left|\x{i} + \frac{\inrad}{2}\right|,\quad \mbox{ and }
  \quad \fn_i(x) = \left|\x{i} - \frac{\inrad}{2}\right|.
\end{align*}
With this choice, the functions $g_\parama(x)$ are minimized at
$x_\parama = -\inrad\parama/2$, and $\inf_{x \in \ConvSet}
g_\parama(x) = cd/2 - c\inrad\err$. Mimicking the previous steps with
$\inrad = 1/2$, we obtain the lower bound
\begin{equation*}
  \disc(g_\parama, g_\paramb) \geq \frac{c\inrad\err}{2}\quad\forall
  \parama \ne \paramb \in \packset.
\end{equation*}
The rest of the proof above did not depend on $\ConvSet$, so that we
again obtain the lower bound $T \geq c_0\, \frac{\dim}{\err^2}$ or $T
\geq \frac{c_0}{\err^2}$ depending on the oracle used, for a
universal constant $c_0$. In this case, the difference in $\disc$
computation means that $\acc = \frac{L \err\inrad}{36} \leq
\frac{L\inrad}{144} $, from which the general claims follow.


\subsection{Proof of Theorem~\ref{ThmStrong}}

We now turn to the proof of lower bounds on the oracle complexity of
the class of strongly convex functions from
Definition~\ref{DefnStrongConvex}. In this case, we work with the
following family of base functions, parametrized by a scalar $\theta
\in [0, 1)$:
\begin{equation}
\label{eqn:base-strong}
  \fp_i(x) = \inrad\theta|\x{i}+\inrad| +
  \frac{(1-\theta)}{4}\left(\x{i} + \inrad\right)^2,~ \quad \mbox{and}
  \quad \fn_i(x) = \inrad\theta|\x{i}-\inrad| +
  \frac{(1-\theta)}{4}\left(\x{i} - \inrad \right)^2.
\end{equation}
A key ingredient of the proof is a uniform lower bound on the
discrepancy $\rho$ between pairs of these functions:
\begin{lemma}
\label{lemma:separation-strong}
Using an ensemble based on the base functions~\eqref{eqn:base-strong},
we have
\begin{align}
\rho(g_\parama, g_\paramb) & \geq \begin{cases} \frac{2 c
    \err^2\inrad^2}{(1-\theta)d} \; \Hamm(\parama, \paramb) & \mbox{if
    $1 - \theta \geq \frac{4\err}{1+2\err}$} \\
\frac{c\err\inrad^2}{d} \; \Hamm(\parama, \paramb) & \mbox{if
  $1-\theta < \frac{4\err}{1+2\err}$.}
\end{cases}
\end{align}
\end{lemma}
\noindent The proof of this lemma is provided in
Appendix~\ref{app:separation-strong}.  Let us now proceed to the
proofs of the main theorem claims.


\paragraph{Part (a)---Proof for $\pval = 1$:} 
We observe that both the functions $\fp_i, \fn_i$ are
$\inrad$-Lipschitz with respect to the $\norm{\cdot}_1$ norm by
construction. Hence, $g_\parama$ is $c\inrad$-Lipschitz and
furthermore, by the definition of Oracle A, we have
\mbox{$\E\norm{\oracsuba{\parama}(x)}_1^2 \leq c^2\inrad^2$.}  In
addition, the function $g_\parama$ is $(1-\theta)c/(4d)$-strongly
convex with respect to the Euclidean norm.  We now follow the same
steps as the proof of Theorem~\ref{ThmConvex}, but this time
exploiting the ensemble formed by the base
functions~\eqref{eqn:base-strong}, and the lower bound on the
discrepancy $\disc(g_\parama, g_\paramb)$ from
Lemma~\ref{lemma:separation-strong}.  We split our analysis into two
sub-cases.\\

\noindent \emph{Case 1:} First suppose that $1-\theta \geq
4\err/(1+2\err)$, in which case Lemma~\ref{lemma:separation-strong}
yields the lower bound
\begin{align*}
\disc(g_\parama, g_\paramb) & \; \geq \; \frac{2 c
  \err^2\inrad^2}{(1-\theta)d} \Hamm(\parama, \paramb) \;
\stackrel{(i)}{\geq} \frac{c \err^2 \inrad^2}{2 (1-\theta)} \quad
\forall \parama \neq \paramb \in \packset,
\end{align*}
where inequality (i) uses the fact that $\Hamm(\parama, \paramb) \geq
d/4$ by definition of $\packset$. Hence by definition of $\errfun$, we
have established the lower bound $\errfun(\err) \geq
\frac{c\err^2\inrad^2}{2(1-\theta)}$. Setting the target error $\acc
\defn c\err^2\inrad^2/(18(1-\theta))$, we observe that this ensures
$\acc \leq \errfun(\err)/9$. Recalling the requirement $\err < 1/4$,
we note that \mbox{$\acc < c\inrad^2/(288(1-\theta))$.} In this
regime, we may apply Lemma~\ref{lemma:redn} to obtain the upper bound
$\P_\truth[\esta{\meth_T} \neq \parama] \leq \frac{1}{3}$. Combining
this upper bound with the lower bound~\eqref{EqnFanoLower} yields the
inequality
\begin{align*}
  \frac{1}{3} & \geq 1 - 2\frac{16T\err^2 + \log 2}{d\log(2/\sqrt{e})}
  \geq 1 - 2\frac{\frac{288T\acc(1-\theta)}{c\inrad^2} + \log
    2}{d\log(2/\sqrt{e})}. 
\end{align*}
Simplifying the above expression yields that for $d \geq 11$, we have
the lower bound
\begin{equation}
  T \geq c\inrad^2\left(\frac{\frac{d}{3}\log(2/\sqrt{e}) - \log
    2}{288\acc(1-\theta)}\right) \geq
  c\inrad^2\frac{d\log(2/\sqrt{e})}{28800\epsilon(1-\theta)}.  
  \label{eqn:strongnearlb}
\end{equation}
Finally, we observe that $L = c\inrad$ and $\strongcon^2 =
(1-\theta)c/(4d)$ which gives $1 - \theta =
4d\inrad\strongcon^2/L$. Substituting the above relations in the lower
bound~\eqref{eqn:strongnearlb} gives the first term in the stated
result for $d \geq 11$.

 To obtain lower bounds for dimensions $d < 11$, we use an argument
 based on $d = 1$.  For this special case, we consider $\fp$ and $\fn$
 to be the two functions of the single coordinate coming out of
 definition~\eqref{eqn:base-strong}. The packing set $\packset$
 consists of only two elements now, corresponding to $\parama = 1$ and
 $\parama = -1$. Specializing the result of
 Lemma~\ref{lemma:separation-strong} to this case, we see that the two
 functions are $2c\err^2\inrad^2/(1-\theta)$ separated. Now we again
 apply Lemma~\ref{lemma:redn} to get an upper bound on the error
 probability and Lemma~\ref{lemma:onecoin} to get a lower bound, which
 gives the result for $d \leq 11$. \\

\noindent \emph{Case 2:} On the other hand, suppose that $1-\theta
\leq 4\err/(1+2\err)$. In this case, appealing to
Lemma~\ref{lemma:separation-strong} gives us that $\disc(g_\parama,
\paramb) \geq c\err\inrad^2/4$ for $\parama \ne \paramb \in
\packset$. Recalling that $L = c\inrad$, we set the desired accuracy
$\acc \defn c\err\inrad^2/36 = L\err\inrad/36$.  From this point
onwards, we mimic the proof of Theorem~\ref{ThmConvex}; doing so
yields that for all $\err \in (0, 1/4)$, we have
\begin{equation*}
T \geq c_0\, \frac{d}{\err^2} = c_0\, \frac{L^2d\inrad^2}{\acc^2},
\end{equation*}
corresponding to the second term in Theorem~\ref{ThmConvex} for a
universal constant $c_0$. \\

\noindent Finally, the third and fourth terms are obtained just like
Theorem~\ref{ThmConvex} by checking the condition $\err < 1/4$ in the
two cases above.  Overall, this completes the proof for the case $p =
1$.


\paragraph{Part (b)---Proof for $\pval > 2$:}

As with the proof of Theorem~\ref{ThmConvex}(b), we use Oracle B that
returns $d$-dimensional values and gradients in this case, with the
base functions defined in equation~\ref{eqn:base-strong}. With this
choice, we have the upper bound
\begin{align*}
\E\|\oracsubb{\parama}(x)\|_p^2 & \leq c^2\dim^{2/\lipexp-2}r^2,
\end{align*}
so that setting the constant $c = L \dim^{1-1/\lipexp}/\inrad$ ensures
that $\E\|\oracsubb{\parama}(x)\|_p^2 \leq L^2$.  As before, we have
the strong convexity parameter
\begin{align*}
\strongcon^2 = \frac{c(1-\theta)}{4d} =
\frac{Ld^{-1/\lipexp}(1-\theta)}{4\inrad},
\end{align*}
Also $\disc(g_\parama, g_\paramb)$ is given by
Lemma~\ref{lemma:separation-strong}. In particular, let us consider
the case $1-\theta \geq 4\err/(1+2\err)$ so that $\errfun(\err) \geq
\frac{c\err^2\inrad^2}{2(1-\theta)}$ and we set the desired accuracy
$\acc \defn \frac{c\err^2\inrad^2}{18(1-\theta)}$ as before. With this
setting of $\acc$, we invoke Lemma~\ref{lemma:redn} as before to argue
that $\P_\truth[\esta{\meth_T} \neq \parama] \leq \frac{1}{3}$. To
lower bound the error probability, we appeal to Lemma~\ref{LemFano}
with $\ell=d$ just like Theorem~\ref{ThmConvex}(b) and obtain the
inequality
\begin{align*}
\frac{1}{3} & \geq 1 - 2\frac{16\, \dim \: T \: \err^2 + \log 2}{\dim
\log(2/\sqrt{e})}.
\end{align*}

Rearranging terms and substituting $\acc =
\frac{c\err^2\inrad^2}{18(1-\theta)}$, we obtain for $d \geq 11$ 
\begin{align*}
  T \geq c_0\, \left(\frac{1}{\delta^2}\right) = c_0\,
  \left(\frac{c\inrad^2}{\acc(1-\theta)}\right), 
\end{align*}
for a universal constant $c_0$. The stated result can now be attained
by recalling $c=L\dim^{1-1/\lipexp}/\inrad$ and $\strongcon^2 =
L\dim^{-1/\lipexp}(1-\theta)/\inrad$ for $1-\theta \geq
4\err/(1+2\err)$ and $d \geq 11$. For $d < 11$, the cases of $\pval >
2$ and $\pval = 1$ are identical up to constant factors in the lower
bounds we state. This completes the proof for $1-\theta \geq
4\err/(1+2\err)$.

Finally, the case for $1-\theta < 4\err/(1+2\err)$ involves similar
modifications as part(a) by using the different expression for
$\disc(g_\parama, g_\paramb)$. Thus we have completed the proof of
this theorem.

\subsection{Proof of Theorem~\ref{ThmSparse}}
\label{SecProofSparse}

We begin by constructing an appropriate subset of $\Fsparse(\kdim)$
over which the Fano method can be applied.  Let $\packset(\kdim) \defn
\{ \parama^1, \ldots, \parama^M \}$ be a set of vectors, such that
each $\parama^j \in \{-1, 0, +1\}^\dim$ satisfies
\begin{equation*}
\|\parama^j\|_0 = \kdim \quad \mbox{for all $j = 1, \ldots, M$,$\:$
  and} \qquad \Hamm(\parama^j, \parama^\ell) \geq \frac{\kdim}{2}
\quad \mbox{for all $j \neq \ell$.}
\end{equation*}
It can be shown that there exists such a packing set with
$|\packset(\kdim)| \geq \exp
\big(\frac{\kdim}{2}\log\frac{d-\kdim}{\kdim/2} \big)$ elements (e.g.,
see Lemma 5 in Raskutti et al.~\cite{raskutti2009minimax}).

For any $\parama \in \packset(\kdim)$, we define the function
\begin{align}
\label{eqn:gdef}
 g_\parama(x) & \defn c \left[\sum_{i=1}^d \left\{\left(\half +
   \parama_i \err \right) \right| \x{i} + \inrad\left| + \left(\half -
   \parama_i \err \right) \left|\x{i}-\inrad \right| \right\} +
   \err\sum_{i=1}^d|\x{i}|\right].
\end{align}
In this definition, the quantity $c > 0$ is a pre-factor to be chosen
later, and $\delta \in (0, \frac{1}{4}]$ is a given error tolerance.
Observe that each function $g_\parama \in \G(\err; \kdim)$ is
  convex, and Lipschitz with parameter $c$ with respect to the
  $\|\cdot\|_\infty$ norm.

Central to the remainder of the proof is the function class $\G(\err;
\kdim) \defn \{g_\parama, \; \parama \in \packset(\kdim) \}$.  In
particular, we need to control the discrepancy $\errfun(\err; \kdim)
\defn \errfun(\G(\err; \kdim))$ for this class.  The following result,
proven in Appendix~\ref{AppLemRhoSparse}, provides a suitable lower
bound:
\begin{lemma}
\label{LemRhoSparse}
We have
 \begin{align}
\errfun(\err; \kdim) \; = \; \inf_{\parama \ne \paramb \in
  \packset(\kdim)} \disc(g_\parama,g_\paramb) & \geq \frac{c \kdim
  \err\inrad}{4}.
 \end{align}
\end{lemma}
\noindent Using Lemma~\ref{LemRhoSparse}, we may complete the proof of
Theorem~\ref{ThmSparse}.  Define the base functions
\begin{equation*}
\fp_i(x) \, \defn \, d\left(\left|\x{i} + \inrad \right| + \err
|\x{i}|\right), \quad \mbox{ and } \quad \fn_i(x) \defn d\left(\left|
\x{i} - \inrad \right| + \err |\x{i}|\right).
\end{equation*}
Consider Oracle B, which returns $\dim$-dimensional gradients based on
the function 
\begin{equation*}
\oracfunb{\parama}(x) = \frac{c}{\dim}\sum_{i=1}^\dim
\big[b_i\fp_i(x) + (1-b_i)\fn_i(x) \big],
\end{equation*}
where $\{b_i\}$ are Bernoulli variables.  By construction, the
function $\oracfunb{\parama}$ is at most $3c$-Lipschitz in
$\ell_\infty$ norm (i.e. $\|\oracsubb{\parama}(x)\|_\infty \leq 3c$),
so that setting $c = \frac{L}{3}$ yields an $L$-Lipschitz function.

Our next step is to use Fano's inequality~\cite{Cover} to lower bound
the probability of error in the multiway testing problem associated
with this stochastic oracle, following an argument similar to (but
somewhat simpler than) the proof of Lemma~\ref{LemFano}.  Fano's
inequality yields the lower bound
\begin{align}
\label{EqnAltFano}
\P [\widehat{\rparam} \neq \Rparam] & \geq 1 -
\frac{\frac{1}{{|\packset| \choose 2 }} \sum_{\parama \neq \paramb}
  \KL{\mprob_\parama}{\mprob_\paramb} + \log 2}{\log |\packset|}.
\end{align}
(As in the proof of Lemma~\ref{LemFano}, we have used convexity of
mutual information~\cite{Cover} to bound it by the average of the
pairwise KL divergences.)  By construction, any two parameters
$\parama,\paramb \in \packset$ differ in at most $2 \kdim$ places, and
the remaining entries are all zeroes in both vectors.  The proof of
Lemma~\ref{LemFano} shows that for $\err \in [0, \frac{1}{4}]$, each
of these $2 \kdim$ places makes a contribution of at most $16
\delta^2$.  Recalling that we have $T$ samples, we conclude that
$\KL{\mprob_\parama}{\mprob_\paramb} \leq 32 \kdim T \err^2$.
Substituting this upper bound into the Fano lower
bound~\eqref{EqnAltFano} and recalling that the cardinality of
$\packset$ is at least $\exp
\big(\frac{\kdim}{2}\log\frac{\dim-\kdim}{\kdim/2} \big)$, we obtain
\begin{align}
\label{eqn:lowerbound2}
\P[\esta{\meth_T} \ne \parama] & \geq 
1 - 2 \left(\frac{32 \kdim T
  \err^2 + \log 2}{\frac{\kdim}{2}\log\frac{\dim-\kdim}{\kdim/2}}
  \right)
\end{align}
By Lemma~\ref{LemRhoSparse} and our choice $c = L/3$, we have
\begin{align*}
\errfun(\err) \geq \frac{c \kdim \err\inrad}{4} \; = \; \frac{L \kdim
  \err\inrad}{12},
\end{align*}
Therefore, if we aim for the target error $\epsilon = \frac{L \kdim
  \err\inrad}{108}$, then we are guaranteed that $\epsilon \leq
\frac{\errfun(\err)}{9}$, as is required for the application of
Lemma~\ref{lemma:redn}. Recalling the requirement $\err \leq 1/4$
gives $\acc \leq \Lval\kdim\err\inrad/432$. Now Lemma~\ref{lemma:redn}
implies that $\P[\esta{\meth_T} \neq \parama] \leq 1/3$, which when
combined with the earlier bound~\eqref{eqn:lowerbound2} yields
\begin{align*}
\frac{1}{3} & \geq 1 - 2 \left(\frac{32 \kdim T \err^2 + \log
  2}{\frac{\kdim}{2}\log\frac{\dim-\kdim}{\kdim/2}} \right).
\end{align*}
Rearranging yields the lower bound
\begin{align*}
T & \geq c_0\, \left(\frac{\log \frac{\dim - \kdim}{\kdim/2}}{\delta^2}
\right) \; = \; c_0 \left ( L^2 \inrad^2\, \kdim^2 \, \frac{\log
\frac{\dim - \kdim}{\kdim/2}}{\epsilon^2} \right),
\end{align*}
for a universal constant $c_0$, where the second step uses the
relation $\err = \frac{108 \epsilon}{L \kdim\inrad}$ for $k,d \geq
11$.  As long as $\kdim \leq \lfloor \dim/2 \rfloor$, we have $\log
\frac{\dim - \kdim}{\kdim/2} = \Theta\left(\log
\frac{\dim}{\kdim}\right)$, which gives the result for $k,d \geq
11$. The result for $k,d \leq 11$ follows Theorem~\ref{ThmConvex}(b)
applied with $p = \infty$, completing the proof. \\

\section{Discussion}
\label{SecDiscussion}

In this paper, we have studied the complexity of convex optimization
within the stochastic first-order oracle model.  We derived lower
bounds for various function classes, including convex functions,
strongly convex functions, and convex functions with sparse optima.
As we discussed, our lower bounds are sharp in general, since there
are matching upper bounds achieved by known algorithms, among them
stochastic gradient descent and stochastic mirror descent.  Our bounds
also reveal various dimension-dependent and geometric aspects of the
stochastic oracle complexity of convex optimization.  An interesting
aspect of our proof technique is the use of tools common in
statistical minimax theory.  In particular, our proofs are based on
constructing packing sets, defined with respect to a pre-metric that
measures how the degree of separation between the optima of different
functions.  We then leveraged information-theoretic techniques, in
particular Fano's inequality and its variants, in order to establish
lower bounds.

There are various directions for future research.  It would be
interesting to consider the effect of memory constraints on the
complexity of convex optimization, or to derive lower bounds for
problems of distributed optimization.  We suspect that the proof
techniques developed in this paper may be useful for studying these
related problems.


\subsection*{Acknowledgements}
AA and PLB gratefully acknowledge partial support from NSF awards
DMS-0707060 and DMS-0830410 and DARPA-HR0011-08-2-0002.  AA was also
supported in part by a Microsoft Research Fellowship.  MJW and PR were
partially supported by funding from the National Science Foundation
(DMS-0605165, and DMS-0907632).  In addition, MJW received funding
from the Air Force Office of Scientific Research (AFOSR-09NL184). We
also thank the anonymous reviewers for helpful suggestions, and
corrections to our results and for pointing out the optimality of our
bounds in the primal-dual norm setting.


\appendix






\section{Proof of Lemma~\ref{lemma:separation-strong}}
\label{app:separation-strong}
Let $g_\parama$ and $g_\paramb$ be an arbitrary pair of functions in
our class, and recall that the constraint set $\ConvSet$ is given by
the ball $\Ball_\infty(r)$. From the definition~\eqref{EqnDefnDisc} of
the discrepancy $\disc$, we need to compute the single function
infimum $\inf_{x \in \Ball_\infty(r)} g_\parama(x)$, as well as the
quantity $\inf_{x \in \Ball_\infty(r)} \{ g_\parama(x) + g_\paramb(x)
\}$.

\paragraph{Evaluating  the single function infimum:}
Beginning with the former quantity, first observe that for any
\mbox{$x \in \Ball_\infty(\inrad)$,} we have
\begin{align}
\label{EqnUsefulAbsolute}
  |\x{i}+\inrad| = \x{i} + \inrad\quad\mbox{and}\quad|\x{i} - \inrad|
  = \inrad - \x{i}.
\end{align}
Consequently, using the definition~\eqref{eqn:base-strong} of the base
functions, some algebra yields the relations
\begin{align*}
\fp_i(x) & = \frac{1-\theta}{4}\x{i}^2 + \frac{1+3\theta}{4}\inrad^2 +
\frac{(1+\theta)}{2}\inrad\x{i}, \qquad \mbox{and} \\
\fn_i(x) & = \frac{1-\theta}{4}\x{i}^2 + \frac{1+3\theta}{4}\inrad^2 -
\frac{(1+\theta)}{2}\inrad\x{i}.
\end{align*}
Using these expressions for $\fp_i$ and $\fn_i$, we obtain
\begin{align*}
\underbrace{\left(\half+\parama_i \err \right)\fp_i(x) +
  \left(\half-\parama_i \err \right)\fn_i(x)}_{h_i(x)} & = \frac{1}{2}
\big( \fp_i(x) + \fn_i(x) \big) + \parama_i \err \big( \fp_i(x) -
\fn_i(x) \big) \\
& = \frac{1-\theta}{4}\x{i}^2 + \frac{1+3\theta}{4}\inrad^2 +
(1+\theta)\parama_i\err\inrad \x{i}.
\end{align*}
A little calculation shows that constrained minimum of the univariate
function $h_i$ over the interval $[-r, r]$ is achieved at
\begin{align*}
\xstar(i) & \defn \begin{cases}
  \frac{-2\parama_i\err\inrad(1+\theta)}{1-\theta} & \mbox{if
    $\frac{1-\theta}{1+\theta} \geq 2\err$} \\
-\parama_i\inrad & \mbox{if $\frac{1-\theta}{1+\theta} < 2\err$,}
\end{cases}
\end{align*}
where we have recalled that $\parama_i$ takes values in $\{-1, +1\}$.
Substituting the minimizing argument $\xstar(i)$, we find that the
minimum value is given by
\begin{align*}
h_i(\xstar(i)) & =  \begin{cases}  \frac{1+3\theta}{4}\inrad^2 -
  \frac{\err^2\inrad^2(1+\theta)^2}{(1-\theta)} &  
\mbox{if 
    $\frac{1-\theta}{1+\theta} \geq 2\err$} \\
\frac{1+\theta}{2}\inrad^2 - (1+\theta)\err\inrad^2 &     \mbox{if
  $\frac{1-\theta}{1+\theta} < 2\err$}. 
\end{cases}
\end{align*}
Summing over all co-ordinates $i \in \{1, 2, \ldots, d\}$, we obtain
\begin{align}
\label{EqnGinf}
  \inf_{x \in \Ball_{\infty}(\inrad)} g_{\parama}(x) & =
  \frac{c}{d}\sum_{i=1}^d h_i(\xstar(i)) \; = \; \begin{cases} 
    -\frac{\err^2\inrad^2c(1+\theta)^2}{(1-\theta)} +
    \frac{c\inrad^2(1+3\theta)}{4} & \mbox{if
      $\frac{1-\theta}{1+\theta} \geq 2\err$} \\
\frac{1+\theta}{2} c \inrad^2 - (1+\theta) c \err \inrad^2 &
\mbox{if 
    $\frac{1-\theta}{1+\theta} < 2\err$.} 
  \end{cases}
\end{align}

\paragraph{Evaluating the joint infimum:}
Here we begin by observing that for any two $\parama, \paramb \in
\packset$, we have
\begin{align}
\label{EqnBasicObserve}
  g_\parama(x) + g_\paramb(x) & = \frac{c}{d} \sum_{i=1}^d \left[
    \frac{1-\theta}{2} \x{i}^2 + \frac{1+3\theta}{2}\inrad^2 +
    2(1+\theta)\parama_i\err\inrad\x{i} \ind(\parama_i =
    \paramb_i)\right].
\end{align}
As in our previous calculation, the only coordinates that contribute
to $\disc(g_\parama, g_\paramb)$ are the ones where $\parama_i \ne
\paramb_i$, and for such coordinates, the function above is minimized
at $\xstar(i) = 0$. Furthermore, the minimum value for any such
coordinate is $(1 + 3 \theta) c \inrad^2 /(2 d)$.

We split the remainder of our analysis into two cases: first, if we
suppose that $\frac{1-\theta}{1+\theta} \geq 2\err$, or equivalently
that $1-\theta \geq 4\err/(1+2\err)$, then
equation~\eqref{EqnBasicObserve} yields that
\begin{align*}
\inf_{x \in \Ball_\infty(r)} \big \{ g_\parama(x) + g_\paramb(x) \} &
= \frac{c}{d}\sum_{i=1}^d \left[\frac{1+3\theta}{2}\inrad^2 -
  \frac{2\err^2\inrad^2(1+\theta)^2}{1-\theta} \ind(\parama_i =
  \paramb_i)\right].
\end{align*}
Combined with our earlier expression~\eqref{EqnGinf} for the single
function infimum, we obtain that the discrepancy is given by
\begin{align*}
  \disc(g_\parama, g_\paramb) & =
  \frac{2\err^2\inrad^2c(1+\theta)^2}{d(1-\theta)} \Hamm(\parama,
  \paramb) \geq \frac{2\err^2\inrad^2c}{d(1-\theta)} \Hamm(\parama,
  \paramb).
\end{align*}

On the other hand, if we assume that $\frac{1-\theta}{1+\theta} <
2\err$, or equivalently that $1-\theta < 4\err/(1+2\err)$, then
we obtain
\begin{align*}
\inf_{x \in \Ball_\infty(r)} \big \{ g_\parama(x) + g_\paramb(x) \} &
= \frac{c}{d}\sum_{i=1}^d \left[\frac{1+3\theta}{2}\inrad^2 -
  \left(2(1+\theta)\inrad^2\err - \frac{1-\theta}{2}\inrad^2\right)
  \ind(\parama_i = \paramb_i)\right],
\end{align*}
Combined with our earlier expression~\eqref{EqnGinf} for the single
function infimum, we obtain
\begin{align*}
  \disc(g_\parama, g_\paramb) =
  \frac{c}{d}\left(2(1+\theta)\inrad^2\err -
  \frac{1-\theta}{2}\inrad^2\right)\Hamm(\parama, \paramb)
  \stackrel{(i)}{\geq} \frac{c(1+\theta)\inrad^2\err}{d}\Hamm(\parama,
  \paramb),
\end{align*}
where step (i) uses the bound $1-\theta < 2 \err (1 + \theta)$. Noting
that $\theta \geq 0$ completes the proof of the lemma.

\section{Proof of Lemma~\ref{LemRhoSparse}}
\label{AppLemRhoSparse}

Recall that the constraint set $\ConvSet$ in this lemma is the ball
$\Ball_\infty(r)$.  Thus, recalling the definition~\eqref{EqnDefnDisc}
of the discrepancy $\disc$, we need to compute the single function
infimum $\inf_{x \in \Ball_\infty(r)} g_\parama(x)$, as well as the
quantity $\inf_{x \in \Ball_\infty(r)} \{ g_\parama(x) + g_\paramb(x)
\}$.

\paragraph{Evaluating  the single function infimum:}
Beginning with the former quantity, first observe that for any
\mbox{$x \in \Ball_\infty(\inrad)$,} we have
\begin{align}
\label{EqnSahand}
\left[ \half + \parama_i \err \right] \; \left| \x{i} + \inrad \right| + \left
    [ \half - \parama_i \err \right] \; \left |\x{i} - \inrad \right| =
    \inrad + 2\parama_i \err \x{i}.
\end{align}
We now consider one of the individual terms arising in the
definition~\eqref{EqnGClassFun} of the function $g_\parama$.  Using
 the relation~\eqref{EqnSahand}, it can be written as
\begin{align*}
\frac{1}{d}\left[\left(\half + \parama_i \err \right)\fp_i(x) +
  \left(\half - \parama_i \err \right) \fn_i(x)\right] & =\left (\half
+ \parama_i\err \right) \left| \x{i} + \inrad \right| + \left(\half -
\parama_i \err \right) \left |\x{i}-\inrad \right| + \err|\x{i}|
\nonumber \\
& = \begin{cases} \inrad + (2 \parama_i + 1) \err \x{i} & \mbox{if
 $\x{i} \geq 0$} \\ \inrad + (2\parama_i-1)\err\x{i} & \mbox{if $\x{i}
     \leq 0$}
 \end{cases}
\end{align*}
From this representation, we see that whenever $\parama_i \neq 0$,
then the $i^{th}$ term in the summation defining $g_\parama$ minimized
at $\x{i} = -\inrad\parama_i$, at which point it takes on its minimum
value $\inrad(1 - \err)$.  On the other hand, for any term with
$\parama_i = 0$, the function is minimized at $\x{i} = 0$ with
associated minimum value of $\inrad$.  Combining these two facts shows
that the vector $-\parama\inrad$ is an element of the set $\arg
\min_{x \in \ConvSet} g_\parama(x)$, and moreover that
\begin{align}
\label{eqn:ming}
 \inf_{x \in \ConvSet} g_\parama(x) & = c\inrad \left( \dim - \kdim
 \err\right).
\end{align}

\paragraph{Evaluating the joint infimum:}
We now turn to the computation of $\inf_{x \in \Ball_\infty(r)}
\{g_\parama(x) + g_\paramb(x) \}$.  From the
relation~\eqref{EqnSahand} and the definitions of $g_\parama$ and
$g_\paramb$, some algebra yields
\begin{align}
\label{eqn:gsum}
\inf_{x \in \ConvSet} \left \{ g_\parama(x) + g_\paramb(x) \right \} & =
c \inf_{x \in \ConvSet} \sum_{i=1}^d \left\{ 2\inrad + 2 \err \left[
  (\parama_i+ \paramb_i) \x{i} + |\x{i}| \right] \right \}.
\end{align}

Let us consider the minimizer of the $i^{th}$ term in this summation.
First, suppose that $\parama_i \neq \paramb_i$, in which case there
are two possibilities.
\begin{itemize}
\item If $\parama_i \neq \paramb_i$ and neither $\parama_i$ nor
  $\paramb_i$ is zero, then we must have $\parama_i + \paramb_i = 0$,
  so that the minimum value of $2\inrad$ is achieved at $\x{i} = 0$.
\item Otherwise, suppose that $\parama_i \neq 0$ and $\paramb_i = 0$.
  In this case, we see from Equation~\eqref{eqn:gsum} that it is
  equivalent to minimizing $\parama_i\x{i} + |\x{i}|$. Setting $\x{i}
  = -\parama_i$ achieves the minimum value \mbox{of $2\inrad$.}
\end{itemize}
In the remaining two cases, we have $\parama_i = \paramb_i$.
\begin{itemize}
\item
If $\parama_i = \paramb_i \ne 0$, then the component is minimized at
$\x{i} = -\parama_i\inrad$ and the minimum value along the component
is $2\inrad(1 - \err)$.
\item If $\parama_i = \paramb_i = 0$, then the minimum value is
  $2\inrad$, achieved at $\x{i} = 0$.
\end{itemize}
Consequently, accumulating all of these individual cases into a single
expression, we obtain
\begin{align}
\label{eqn:mingsum}
\inf_{x \in \ConvSet} \left \{ g_\parama(x) + g_\paramb(x) \right \} & =
2c\inrad \, \left(\dim - \err\sum_{i=1}^\dim \ind[\parama_i = \paramb_i \ne
  0] \right).
\end{align}

Finally, combining equations~\eqref{eqn:ming} and~\eqref{eqn:mingsum}
in the definition of $\disc$, we find that
\begin{align*}
\disc(g_\parama,g_\paramb) &= 2c\inrad \left[d - \err\sum_{i=1}^\dim
  \ind[ \parama_i = \paramb_i \ne 0] - (\dim - \kdim \err ) \right]
\\ 
&= 2c\err\inrad \left[\kdim - \sum_{i=1}^\dim\ind[ \parama_i =
    \paramb_i \ne 0] \right] \\
& = c\inrad \err\Hamm(\parama,\paramb),
\end{align*}
where the second equality follows since $\parama$ and $\paramb$ have
exactly $\kdim$ non-zero elements each. Finally, since $\packset$ is an
$\kdim/2$-packing set in Hamming distance, we have $\Hamm(\parama,
\paramb) \geq \kdim/2$, which completes the proof. 

\section{Upper bounds via mirror descent}
\label{AppMirror}
 
This appendix is devoted to background on the family of mirror descent
methods.  We first describe the basic form of the algorithm and some
known convergence results, before showing that different forms of
mirror descent provide matching upper bounds for several of the lower
bounds established in this paper, as discussed in the main text.

\subsection{Background on mirror descent}

Mirror descent is a generalization of (projected) stochastic gradient
descent, first introduced by Nemirovski and Yudin~\cite{yudin83book};
here we follow a more recent presentation of it due to Beck and
Teboulle~\cite{beck2003mirror}.  For a given norm $\norm{\cdot}$, let
$\prox:\real^d \rightarrow \real \cup \{+\infty\}$ be a differentiable
function that is $1$-strongly convex with respect to $\norm{\cdot}$,
meaning that
\begin{align*}
\prox(y) & \geq \prox(x) + \inprod{\grad \prox(x)}{y-x} + \half
\norm{y-x}^2.
\end{align*}
We assume that $\prox$ is a function of Legendre
type~\cite{Rockafellar, Hiriart1}, which implies that the conjugate
dual $\prox^*$ is differentiable on its domain with $\grad \prox^* =
\big(\grad \prox \big)^{-1}$.  For a given proximal function, we let
$D_\prox$ be the Bregman divergence induced by $\prox$, given by
\begin{align}
D_\prox(x,y) & \defn \prox(x) - \prox(y) - \inprod{\grad
  \prox(y)}{x-y}.
\end{align}
With this set-up, we can now describe the mirror descent algorithm
based on the proximal function $\prox$ for minimizing a convex
function $f$ over a convex set $\set$ contained within the domain of
$\prox$.  Starting with an arbitrary initial $x_0 \in \set$, it
generates a sequence $\{x_t\}_{t=0}^\infty$ contained within $\set$
via the updates
\begin{align}
\label{EqnMirrorUp}
x_{t+1} & = \arg \min_{x \in \set} \big \{\eta_t \inprod{x}{\grad
  f(x_t)} + D_\prox(x,x_t) \big\},
\end{align}
where $\eta_t > 0$ is a stepsize. In case of stochastic optimization,
$\nabla f(x_t)$ is simply replaced by the noisy version
$\oracleg(x_t)$.

A special case of this algorithm is obtained by choosing the proximal
function $\prox(x) = \half\|x\|_2^2$, which is $1$-strongly convex
with respect to the Euclidean norm.  The associated Bregman divergence
$D_\prox(x,y) = \half\|x-y\|_2^2$ is simply the Euclidean norm, so
that the updates~\eqref{EqnMirrorUp} correspond to a standard
projected gradient descent method.  If one receives only an unbiased
estimate of the gradient $\nabla f(x_t)$, then this algorithm
corresponds to a form of projected stochastic gradient descent.
Moreover, other choices of the proximal function lead to different
stochastic algorithms, as discussed below.

Obtaining explicit convergence rates for this algorithm can be
obtained under appropriate convexity and Lipschitz assumptions for
$f$.  Following the set-up used in our lower bound analysis, we assume
that $\E\dnorm{\grad \oracleg(x_t)}^2 \leq L^2$ for all $x \in \set$,
where $\dnorm{v} \defn \sup_{\norm{x} \leq 1} \inprod{x}{v}$ is the
dual norm defined by $\norm{\cdot}$. Given stochastic mirror descent
based on unbiased estimates of the gradient, it can be showed that
(see e.g., Chapter 5.1 of NY~\cite{yudin83book} or Beck and
Teboulle~\cite{beck2003mirror}) with the initialization $x_0 =
\arg\min_{x \in \Sset} \prox(x)$ and stepsizes $\eta_t = 1/\sqrt{t}$,
the optimization error of the sequence $\{x_t\}$ is bounded as
\begin{align}
\label{EqnBeckTeb}
\frac{1}{T}\sum_{t=1}^T \E \big[f(x_t) - f(x^*)\big] & \leq
L\sqrt{\frac{D_\prox(x^*,x_1)}{T}} \; \leq \;
L\sqrt{\frac{\prox(x^*)}{T}}
\end{align}

Note that this averaged convergence is a little different from the
convergence of $x_T$ discussed in our lower bounds. In order to relate
the two quantities, observe that by Jensen's inequality
\begin{align*}
  \E \left [f\left(\frac{\sum_{t=1}^Tx_t}{T} \right) \right] & \leq
  \frac{1}{T} \E \big[f(x_t) \big].
\end{align*}
Consequently, based on mirror descent for $T-1$ rounds, we may set
$x_T = \frac{1}{T-1}\sum_{t=1}^{T-1}x_t$ so as to obtain the same
convergence bounds up to constant factors. In the following
discussion, we assume this choice of $x_T$ for comparing the mirror
descent upper bounds to our lower bounds.


\subsection{Matching upper bounds}

\newcommand{\specp}{\ensuremath{a}}

Now consider the form of mirror descent obtained by choosing the
proximal function 
\begin{align}
\label{EqnSpecProx}
\prox_\specp(x) & \defn \frac{1}{(\specp-1)}\norm{x}_\specp^2 \qquad
\mbox{for $1 < \specp \leq 2$.}
\end{align}
Note that this proximal function is $1$-strongly convex with respect
to the $\ell_\specp$-norm for $1 < \specp \leq 2$, meaning that
\begin{align*}
  \frac{1}{(\specp-1)}\norm{x}_\specp^2 & \geq
  \frac{1}{(\specp-1)}\norm{y}_\specp^2 +
  \biggl(\grad\frac{1}{(\specp-1)}\norm{x}_\specp^2\bigg)^T (x-y) +
  \half\norm{x-y}_\specp^2.
\end{align*}

\paragraph{Upper bounds for dual setting:} Let us start from the case
$1 \leq p \leq 2$. In this case we use stochastic gradient descent
with , and the choice of $p$ ensures that $\E\norm{\oracleg(x)}^2_2
\leq \E\norm{\oracleg(x)}^2_p \leq L^2$ (the second inequality is true
by assumption of Theorem~\ref{ThmConvex}). Also a straightforward
calculation shows that $\norm{x^*}_2 \leq \norm{x^*}_qd^{1/2-1/q}$ so
that we get the upper bound:
\begin{align*}
  \E\left[f(x_T) - f(x^*)\right] &=
  \order\left(\frac{Ld^{1/2-1/q}}{\sqrt{T}}\right), 
\end{align*}
which matches the lower bound from Equation~\eqref{EqnDualLB} for this
case. For $p \geq 2$, we use mirror descent with $a = q = p/(p-1)$. In
this case, $\E\norm{\oracleg(x)}^2_p \leq L^2$ and $\norm{x^*}_q \leq
1$ for the convex set $\Ball_q(1)$ and the function class
$\Fconv(\Ball_q(1), L, p)$. Hence in this case, the upper bound from
Equation~\ref{EqnBeckTeb} is $\order(L/\sqrt{T})$ as long as $p =
o(\log d)$, which again matches our lower bound from
Equation~\ref{EqnDualLB}. Finally, for $p = \Omega(\log d)$, we use
mirror descent with $a = 2\log d/(2\log d - 1)$, which gives an upper
bound of $\order(L\sqrt{\log d/T})$ (since $1/(a-1) = \order(\log d)$
in this regime).

\paragraph{Upper bounds for $\ell_\infty$ ball:}
 For this case, we use mirror descent based on the proximal function
 $\Prox_\specp$ with $\specp = q$.  Under the condition
 $\norm{\xstar}_\infty \leq 1$, a condition which holds in our lower
 bounds, we obtain
\begin{align*}
  \norm{\xstar}_q & \leq \|\xstar\|_\infty \; \dim^{1/q} \; = \;
  \dim^{1/q},
\end{align*}
which implies that $\Phi_q(\xstar) = \order(\dim^{2/q})$.  Under the
conditions of Theorem~\ref{ThmConvex}, we have
$\E\norm{\oracleg(x_t)}^2_p \leq L^2$ where $p = q/(q-1)$ defines the
dual norm. Note that the condition $1 < q \leq 2$ implies that $p \geq
2$.  Substituting this in the upper bound~\eqref{EqnBeckTeb} yields
\begin{align*}
  \E \big[ f(x_T) - f(x^*) \big] & = \order \bigg( L \;
  \sqrt{\dim^{2/q}/T} \bigg) \; = \; \order \bigg( L \dim^{1-1/p}
  \sqrt{\frac{1}{T}} \bigg),
\end{align*}
which matches the lower bound from Theorem~\ref{ThmConvex}(b) (we note
that there is an additional log factor here just like the preceding
discussion when $p = \order(\log d)$ which we ignore here).

For $1 \leq p \leq 2$, we use stochastic gradient descent with $q =
2$, in which case $\norm{x^*}_2 \leq \sqrt{d}$ and
$\E\norm{\oracleg(x_t)}_2^2 \leq \E\norm{\oracleg(x_t)}_p^2 \leq L^2$
by assumption. Substituting these in the upper bound for mirror
descent yields an upper bound to match the lower bound of
Theorem~\ref{ThmConvex}(a).

\paragraph{Upper bounds for Theorem~\ref{ThmSparse}:}

In order to recover matching upper bounds in this case, we use the
function $\prox_\specp$ from Equation~\eqref{EqnSpecProx} with 
$\specp = \frac{2 \log \dim}{2\log\dim - 1}$. In this case, the
resulting upper bound~\eqref{EqnBeckTeb} on the convergence rate takes
the form
\begin{align}
\label{eqn:sparserate}
\E \left[ f(x_T) - f(x^*) \right] \; = \; \order \left
(L\sqrt{\frac{\|x^*\|_\specp^2}{2 (\specp-1) T }} \right) & = \order
\left(L\sqrt{\frac{\|x^*\|_\specp^2 \log \dim}{T}} \right),
\end{align}
since $\frac{1}{a - 1} = 2 \log \dim - 1$. Based on the conditions of
Theorem~\ref{ThmSparse}, we are guaranteed that $\xstar$ is
$\kdim$-sparse, with every component bounded by $1$ in absolute value,
so that $\norm{\xstar}_\specp^2 \leq \kdim^{2/\specp} \leq \kdim^2$,
where the final inequality follows since $\specp > 1$. Substituting
this upper bound back into Equation~\eqref{eqn:sparserate} yields
\begin{align*}
  \E \left[f(x_T) - f(x^*) \right] & = \order \left
  (L\sqrt{\frac{\kdim^2 \log \dim}{T}} \right).
\end{align*}
Note that whenever $\kdim = \HackoRama$ for some $\delta > 0$, then we
have $\log \dim = \Theta(\log \frac{\dim}{\kdim})$, in which case this
upper bound matches the lower bound from Theorem~\ref{ThmSparse} up to
constant factors, as claimed.


\bibliographystyle{IEEEtran} 
\bibliography{IEEEabrv,optimization}

\end{document}